%% file: preprint.tex
\documentclass{applemlr}

\input{apple_preamble}

\usepackage[utf8]{inputenc}
\usepackage{newunicodechar}
\newunicodechar{�}{\tiny$\blacksquare$}
\usepackage{nicefrac}
\usepackage{ragged2e}
\usepackage{pgfplots}
\pgfplotsset{compat=1.18}
\usepgfplotslibrary{fillbetween}
\usetikzlibrary{calc,arrows.meta,fadings}
\usepackage{tikz-3dplot}
\usepackage{inconsolata}
\usepackage{bbm}
\usepackage{thm-restate}

\input{applecolors}

\input{maths}

\title{Scaling Categorical Flow Maps}

\author[*]{Oscar Davis}
\author{Anastasiia Filippova}
\author{Pierre Ablin}
\author{Victor Turrisi}
\author{Amitis Shidani}
\author{Marco Cuturi}
\author{Louis B\'ethune}

\affiliation{Apple, University of Oxford*}

\abstract{Continuous diffusion and flow matching models could represent a powerful alternative to autoregressive approaches for language modelling (LM), as they unlock a host of advantages currently reserved for continuous modalities, including accelerated sampling and tilting. Recently, several works have demonstrated the possibility of generating discrete data continuously by a simple flow matching process between a Gaussian and the one-hot encoded data distribution. They have further shown the feasibility of accelerated sampling via Categorical Flow Maps (CFMs), resulting in competitive sample quality in the few-step regime. However, this method had only been evaluated at relatively modest scales ($<1$B), leaving the question of its scalability completely open. In this article, we train a $1.7$B-parameter base flow model on $2.1$T tokens and self-distill it into a CFM that generates diverse, high-quality text in as few as $4$ inference steps while maintaining near-data-level token entropy. Furthermore, we introduce a likelihood bound for CFMs in the semi-discrete setting, and show that they can be used to score the model on standard LM benchmarks, achieving results in the same range as discrete diffusion methods. Finally, we uncover some of the challenges that arise from training these models at scale, and we provide prescriptive insights on loss weighting and time scheduling.}

\metadata[Correspondence]{\sffamily Oscar Davis: \href{mailto:oscar.davis@cs.ox.ac.uk}{oscar.davis@cs.ox.ac.uk}, Louis B\'ethune: \href{mailto:l_bethune@apple.com}{l\_bethune@apple.com}}
\metadata[Contributions]{\sffamily Work done while OD was an intern at Apple}
\date{\sffamily May 8, 2026}

\newtheorem{theorem}{Theorem}[section]
\newtheorem{lemma}[theorem]{Lemma}
\newtheorem{proposition}[theorem]{Proposition}

\definecolor{propbg}{RGB}{241,239,248}
\definecolor{propaccent}{HTML}{2B3D50}
\tcolorboxenvironment{proposition}{
  enhanced, breakable,
  colback=propbg, colframe=propbg,
  boxrule=0pt, arc=8pt,
  left=10pt, right=10pt, top=8pt, bottom=8pt,
}
\tcolorboxenvironment{theorem}{
  enhanced, breakable,
  colback=propbg, colframe=propbg,
  boxrule=0pt, arc=8pt,
  left=10pt, right=10pt, top=8pt, bottom=8pt,
}
\tcolorboxenvironment{definition}{
  enhanced, breakable,
  colback=propbg, colframe=propbg,
  boxrule=0pt, arc=8pt,
  left=10pt, right=10pt, top=8pt, bottom=8pt,
}

\newtcolorbox{propbox}{
  enhanced, breakable,
  colback=propbg, colframe=propbg,
  boxrule=0pt, arc=8pt,
  left=10pt, right=10pt, top=8pt, bottom=8pt,
  before upper={},
}

\definecolor{samplebg}{RGB}{254,251,244}
\definecolor{sampleframe}{RGB}{208,202,229}
\definecolor{sampletitle}{RGB}{62,46,101}
\tcbset{
  samplestyle/.style={
    enhanced,
    colback=samplebg,
    colframe=samplebg,
    coltitle=sampletitle,
    colbacktitle=ApplePurple2!40!white,
    fonttitle=\sffamily\bfseries,
    boxrule=0pt,
    arc=6pt,
    outer arc=6pt,
    left=6pt,
    right=6pt,
    top=6pt,
    bottom=6pt,
    toptitle=6pt,
    bottomtitle=6pt,
    titlerule=0.8pt,
    titlerule style={sampleframe},
  }
}
\newtcolorbox{samplebox}[1][]{
  samplestyle,
  before upper={\small\ttfamily\justifying\rightskip=0pt plus 2em\emergencystretch=2em\hyphenpenalty=0\exhyphenpenalty=0\parskip=8pt},
  title={#1},
}

\newtcolorbox{promptbox}[1][PROMPT]{
  enhanced,
  colback=ApplePurple2!40!white,
  colframe=ApplePurple2!40!white,
  boxrule=0pt,
  arc=3pt,
  left=6pt, right=6pt, top=4pt, bottom=4pt,
  fontupper=\small\ttfamily,
  before skip=5.5pt, after skip=5.5pt,
  title={#1},
  fonttitle=\footnotesize\sffamily\bfseries,
  coltitle=ApplePurple7,
  colbacktitle=ApplePurple2!40!white,
  attach title to upper={\\[2pt]},
}
\newtcolorbox{completionbox}[1][COMPLETION]{
  enhanced,
  colback=AppleYellow2!60!white,
  colframe=AppleYellow2!60!white,
  boxrule=0pt,
  arc=3pt,
  left=6pt, right=6pt, top=4pt, bottom=4pt,
  fontupper=\small\ttfamily,
  before skip=5.5pt, after skip=5.5pt,
  title={#1},
  fonttitle=\footnotesize\sffamily\bfseries,
  coltitle=AppleYellow7,
  colbacktitle=AppleYellow2!60!white,
  attach title to upper={\\[2pt]},
}


\tcbset{
  vrbboxbase/.style={
    enhanced,
    colback=AppleYellow2!60!white,
    colframe=AppleYellow2!60!white,
    boxrule=0pt,
    arc=3pt,
    left=6pt, right=6pt, top=4pt, bottom=4pt,
    before skip=5.5pt, after skip=5.5pt,
    fonttitle=\footnotesize\sffamily\bfseries,
    coltitle=AppleYellow7,
    colbacktitle=AppleYellow2!60!white,
    attach title to upper={\par\vspace{1.5pt}},
  }
}
\newenvironment{nfebox}[3]{%
  \def\tcbvrbtitle{NFE #1 \hfill Entropy: #2 \quad Gen PPL: #3}%
  \tcbverbatimwrite{\jobname.vrb}%
}{%
  \endtcbverbatimwrite%
  \begin{tcolorbox}[vrbboxbase, title={\tcbvrbtitle}, height from=0pt to 4cm]%
  \begingroup
  \footnotesize\ttfamily\raggedright
  \setlength{\parindent}{0pt}%
  \setlength{\parskip}{0pt}%
  \obeylines\obeyspaces\input{\jobname.vrb}%
  \endgroup
  \end{tcolorbox}%
}
\newenvironment{nfeboxfull}[3]{%
  \def\tcbvrbtitle{NFE #1 \hfill Entropy: #2 \quad Gen PPL: #3}%
  \tcbverbatimwrite{\jobname.vrb}%
}{%
  \endtcbverbatimwrite%
  \begin{tcolorbox}[vrbboxbase, title={\tcbvrbtitle}]%
  \begingroup
  \footnotesize\ttfamily\raggedright
  \setlength{\parindent}{0pt}%
  \setlength{\parskip}{0pt}%
  \obeylines\obeyspaces\input{\jobname.vrb}%
  \endgroup
  \end{tcolorbox}%
}
\newenvironment{deltapplbox}[2]{%
  \def\tcbvrbtitle{NFE #1 \hfill \deltappl: #2}%
  \tcbverbatimwrite{\jobname.vrb}%
}{%
  \endtcbverbatimwrite%
  \begin{tcolorbox}[vrbboxbase, title={\tcbvrbtitle}, height from=0pt to 4cm]%
  \begingroup
  \footnotesize\ttfamily\raggedright
  \setlength{\parindent}{0pt}%
  \setlength{\parskip}{0pt}%
  \obeylines\obeyspaces\input{\jobname.vrb}%
  \endgroup
  \end{tcolorbox}%
}
\newenvironment{methodbox}[2]{%
  \def\tcbvrbtitle{#1 \hfill \deltappl: #2}%
  \tcbverbatimwrite{\jobname.vrb}%
}{%
  \endtcbverbatimwrite%
  \begin{tcolorbox}[vrbboxbase, title={\tcbvrbtitle}, height from=0pt to 5cm]%
  \begingroup
  \footnotesize\ttfamily\raggedright
  \setlength{\parindent}{0pt}%
  \setlength{\parskip}{0pt}%
  \obeylines\obeyspaces\input{\jobname.vrb}%
  \endgroup
  \end{tcolorbox}%
}

\begin{document}

\maketitle
\input{sections/introduction}
\input{sections/background}
\input{sections/method}
\input{sections/experiments}
\input{sections/related}
\input{sections/limitations}
\input{sections/conclusion}

\bibliographystyle{apalike}
\bibliography{ref}

\clearpage
\appendix
\crefalias{section}{appendix}
\crefalias{subsection}{appendix}
\crefalias{subsubsection}{appendix}

\input{sections/appendix}

\end{document}

%% file: apple_preamble.tex
\usepackage{amsmath}
\usepackage{enumerate}
\usepackage{algorithm}
\usepackage{algpseudocode}
\usepackage{amsfonts}
\usepackage{amsthm}
\usepackage{cleveref}
\usepackage{diagbox}
\usepackage{colortbl}
\usepackage{amssymb}
\usepackage{xspace}
\usepackage{wrapfig}
\usepackage{adjustbox}
\usepackage{tabularx}
\usepackage{booktabs}
\usepackage{mathtools}
\usepackage{tikz}
\usepackage{enumitem}
\usepackage{silence}
\usepackage{dsfont}
\usepackage[table]{xcolor}
\usepackage[dvipsnames]{xcolor}
\usepackage{multirow}
\usepackage{makecell}
\usepackage{xfakebold}

\definecolor{textgray}{HTML}{6E6E73}
\usetikzlibrary{positioning, calc}
\usetikzlibrary{decorations.pathmorphing}

\makeatletter
\patchcmd{\wrong@fontshape}{\@gobbletwo}{}{}{}
\makeatother
\numberwithin{equation}{section}
\makeatletter
\AtBeginDocument{
  \urlstyle{sf}
  
}
\makeatother

\definecolor{light}{RGB}{125, 125, 125}
\crefname{tcb@cnt@pbox}{code}{code}
\Crefname{tcb@cnt@pbox}{Code}{Code}
\crefname{assumption}{assumption}{assumption}
\Crefname{assumption}{Assumption}{Assumptions}

\newtcolorbox[auto counter]{pbox}[2][]{
  colback=white,
  title=Code~\thetcbcounter: #2,
  #1,fonttitle=\sffamily,
  fontupper=\sffamily,
  arc=2pt,
  colframe=bgcolor,
  coltitle=fgcolor,
  colbacktitle=bgcolor,
  toptitle=0.25cm,
  bottomtitle=0.125cm
}

\makeatletter
\newcommand\applefootnote[1]{%
  \begingroup
  \renewcommand\thefootnote{}%
  \renewcommand\@makefntext[1]{\noindent##1}%
  \footnote{#1}%
  \addtocounter{footnote}{-1}%
  \endgroup
}
\makeatother

\definecolor{cverbbg}{gray}{0.90}

%% file: applecolors.tex
\definecolor{AppleWhite}{RGB}{255,255,255}
\definecolor{ApplePrimaryCoolGray}{RGB}{116,128,139}
\definecolor{AppleCoolGray1}{RGB}{199,209,214}
\definecolor{AppleCoolGray2}{RGB}{147,174,190}
\definecolor{AppleCoolGray3}{RGB}{124,147,160}
\definecolor{AppleCoolGray4}{RGB}{92,102,109}
\definecolor{AppleCoolGray5}{RGB}{78,93,100}
\definecolor{AppleCoolGray6}{RGB}{53,60,65}
\definecolor{AppleBlack}{RGB}{0,0,0}
\definecolor{AppleSecondaryChartGray}{RGB}{168,168,168}
\definecolor{AppleChartGray2}{RGB}{233,233,233}
\definecolor{AppleChartGray3}{RGB}{211,211,211}
\definecolor{AppleChartGray4}{RGB}{190,190,190}
\definecolor{AppleChartGray5}{RGB}{140,140,140}
\definecolor{AppleChartGray6}{RGB}{102,102,102}
\definecolor{AppleChartGray7}{RGB}{64,64,64}
\definecolor{ApplePrimaryChartBlue}{RGB}{84,151,193}
\definecolor{AppleBlue2}{RGB}{212,229,239}
\definecolor{AppleBlue3}{RGB}{169,202,223}
\definecolor{AppleBlue4}{RGB}{127,177,209}
\definecolor{AppleBlue5}{RGB}{71,130,166}
\definecolor{AppleBlue6}{RGB}{55,99,128}
\definecolor{AppleBlue7}{RGB}{45,72,89}
\definecolor{ApplePrimaryChartGreen}{RGB}{83,172,121}
\definecolor{AppleGreen2}{RGB}{212,234,221}
\definecolor{AppleGreen3}{RGB}{169,213,188}
\definecolor{AppleGreen4}{RGB}{126,193,155}
\definecolor{AppleGreen5}{RGB}{58,140,82}
\definecolor{AppleGreen6}{RGB}{39,102,54}
\definecolor{AppleGreen7}{RGB}{29,58,31}
\definecolor{ApplePrimaryChartYellow}{RGB}{253,195,93}
\definecolor{AppleYellow2}{RGB}{254,240,214}
\definecolor{AppleYellow3}{RGB}{254,224,174}
\definecolor{AppleYellow4}{RGB}{254,210,134}
\definecolor{AppleYellow5}{RGB}{230,168,69}
\definecolor{AppleYellow6}{RGB}{191,131,46}
\definecolor{AppleYellow7}{RGB}{153,107,54}
\definecolor{ApplePrimaryChartOrange}{RGB}{250,151,92}
\definecolor{AppleOrange2}{RGB}{254,229,214}
\definecolor{AppleOrange3}{RGB}{252,203,173}
\definecolor{AppleOrange4}{RGB}{252,178,133}
\definecolor{AppleOrange5}{RGB}{227,121,68}
\definecolor{AppleOrange6}{RGB}{191,87,46}
\definecolor{AppleOrange7}{RGB}{143,59,36}
\definecolor{ApplePrimaryChartRed}{RGB}{227,94,105}
\definecolor{AppleRed2}{RGB}{248,215,217}
\definecolor{AppleRed3}{RGB}{241,174,180}
\definecolor{AppleRed4}{RGB}{234,135,143}
\definecolor{AppleRed5}{RGB}{196,63,77}
\definecolor{AppleRed6}{RGB}{153,35,53}
\definecolor{AppleRed7}{RGB}{102,19,43}
\definecolor{ApplePrimaryChartPurple}{RGB}{161,150,204}
\definecolor{ApplePurple2}{RGB}{231,228,242}
\definecolor{ApplePurple3}{RGB}{208,202,229}
\definecolor{ApplePurple4}{RGB}{185,176,217}
\definecolor{ApplePurple5}{RGB}{128,113,171}
\definecolor{ApplePurple6}{RGB}{89,76,128}
\definecolor{ApplePurple7}{RGB}{62,46,101}
\definecolor{AppleCoolGray}{RGB}{116,128,139}
\definecolor{AppleChartGray}{RGB}{168,168,168}
\definecolor{AppleBlue}{RGB}{84,151,193}
\definecolor{AppleGreen}{RGB}{83,172,121}
\definecolor{AppleYellow}{RGB}{253,195,93}
\definecolor{AppleOrange}{RGB}{250,151,92}
\definecolor{AppleRed}{RGB}{227,94,105}
\definecolor{ApplePurple}{RGB}{161,150,204}

%% file: maths.tex
\newcommand{\R}{\mathbb{R}}
\newcommand{\N}{\mathbb{N}}
\newcommand{\E}{\mathbb{E}}
\newcommand{\diff}{\mathrm{d}}
\newcommand{\stopgrad}{\mathop{\mathrm{sg}}}
\newcommand{\argmax}{\mathop{\mathrm{argmax}}}
\newcommand{\KL}[2]{\operatorname{KL}\!\left(#1 \,\middle\|\, #2\right)}
\newcommand{\JSD}[2]{\operatorname{JSD}\!\left(#1 \,\middle\|\, #2\right)}

\newcommand{\mauve}{MAUVE$\star$\xspace}
\newcommand{\deltappl}{$\Delta$ PPL\xspace}

%% file: sections/introduction.tex
\section{Introduction}
Although in language modelling the dominant paradigm has remained autoregressive (AR), diffusion models~\citep{sohldickstein2015deepunsupervisedlearning,ddpm} and flow matching~\citep{flowmatching,stochasticinterpolants} were responsible for large breakthroughs on continuous modalities. Driven by their success, a line of work has sought to extend diffusion to discrete distributions, an approach coined \emph{discrete diffusion}~\citep{d3pm,argmaxflows,campbell2022ctmc,mdlm,shi2024maskeddiffusion,discreteflowmatching}. The hope of these models was to extend the rich and principled diffusion toolbox to discrete data, including guidance~\citep{cfg}, distribution tilting~\citep{adjointmatching}, and accelerated sampling~\citep{consistencymodels,song2023improvedtechniquestrainingconsistency}. However, discrete diffusion methods evolve on the discrete grid, preventing the usage of the well-established continuous techniques.

\begin{figure}[t]
    \centering
    \includegraphics[width=\textwidth]{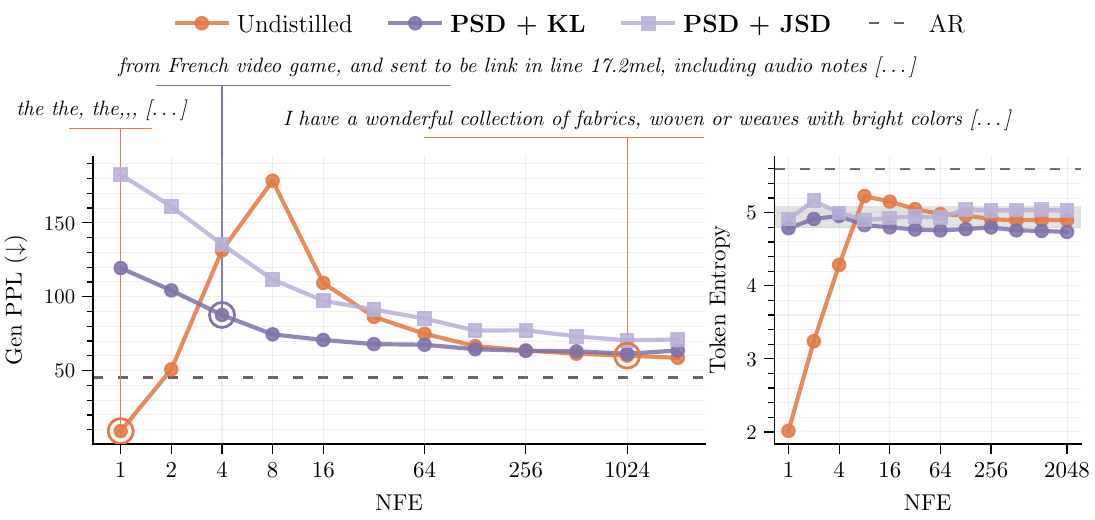}
    \caption{(Left) Gen PPL vs. NFE. (Right) Token entropy vs. NFE. The gray band indicates the token entropy of the data distribution. Above the circled points are provided short qualitative samples that reflect the generation quality at that NFE regime. Self-distillation enables high-quality, diverse generation at significantly fewer function evaluations.}
    \label{fig:sd_flowmap_pareto}
\end{figure}

A parallel line of work developed continuous-state processes to generate discrete data, either on the simplex~\citep{dirichletflowmatching,fisherflowmatching,categoricalflowmatching} or directly in $\R^d$~\citep{vfm}. Recently,~\citet{cfm,fmlm,dfm} scaled these ideas to common discrete diffusion benchmarks, and showed that these methods are competitive, but also succeeded in porting flow map matching~\citep{flowmapmatching,howtobuildconsistencymodel} to discrete data, enabling high quality few-step sampling. We refer to these methods as \emph{Categorical Flow Maps} (CFMs). If successful at scale, CFMs could unlock the flow map toolbox for steering and fine-tuning~\citep{diamondmaps,metaflowmaps}, trading compute for sample quality rather than merely generating plausible outputs.

Nonetheless, it remains unclear whether CFMs can scale to the billion-parameter, trillion-token regime. For a token vocabulary $V$, and context length $L$, they require materialising $L \times \lvert V\rvert$ logit matrices, where $\lvert V\rvert$ can easily attain $200{,}000$ tokens in modern LLMs, imposing a significant memory and computational overhead.

In this work, we show that CFMs scale competitively to this setting. Our contributions are as follows:
\begin{itemize}
    \item We explore the CFM design space via $350$+ configurations at $400$M proxy scale, identifying robust hyperparameter settings that inform the full-scale run (\Cref{sec:method}).
    \item We train a $1.7$B-parameter flow matching model on $2.1$T tokens, directly comparing to masked and uniform diffusion~\citep{scalingbeyondmaskeddiffusion} under matched compute (\Cref{sec:experiments}).
    \item We self-distill the pretrained flow into a CFM, improving the quality vs. NFE Pareto front in the few-step regime (\Cref{fig:sd_flowmap_pareto}).
    \item To enable likelihood-based evaluation, we provide a semi-discrete ELBO, and report results on MCQA benchmarks in similar ranges as other non-autoregressive methods (\Cref{sec:elbo}).
\end{itemize}
While we uncover some of the limitations of the method as well, we demonstrate the viability of CFMs at this scale.

%% file: sections/background.tex
\section{Background}\label{sec:background}
We denote by $V$ the vocabulary of our target distribution, $p_\mathrm{data}$, which has support $V^L$ for a context length $L \in \N^\star$. For parameters $\theta$, $p_\theta$ refers to our model distribution.  

\subsection{Autoregressive Modelling}
AR models approach language modelling by decomposing the joint distribution over a sequence into a product of conditional distributions using the chain rule of probability:
\begin{equation}
    p_{\mathrm{data}}(x_{1:L})
    =
    \prod_{i=1}^{L}
    p(x_i \mid x_{<i}),
\end{equation}
where $x_{<i} = (x_1,\dots,x_{i-1})$ denotes the prefix preceding token $x_i$. This factorisation converts the problem of modelling a sequence into a next-token prediction problem. A parametric model $p_\theta$ is trained to approximate these conditional distributions by maximising the likelihood of observed sequences in the training corpus.

\subsection{Stochastic Interpolants}
\begin{wrapfigure}{r}{0.55\textwidth}
\vspace{-4em}
\centering
\includegraphics[width=\linewidth,trim=0 35 0 25,clip]{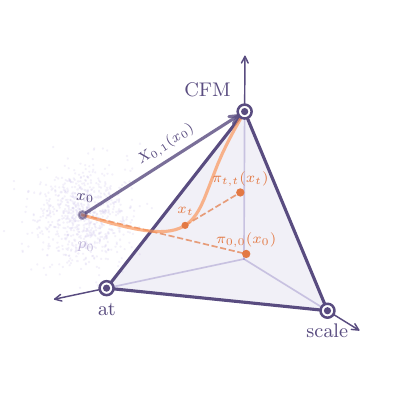}
\caption{Categorical Flow Maps on $\Delta^{d}$, each vertex corresponding to a token. The \textcolor{AppleOrange5}{ODE trajectory} continuously transports noisy $x_0$ towards data. The flow map \textcolor{ApplePurple6}{$X_{0,1}(x_0)$} traverses the path in one step.}
\label{fig:simplex_cfm}
\vspace{0.8em}
\end{wrapfigure}
We model the generative problem as a dynamic transport of measure~\citep{flowmatching,stochasticinterpolants,rectifiedFlow}. Starting from a prior $p_0$, we evolve the distribution through the following ordinary differential equation (ODE) on the sample trajectory level:
\begin{equation}
    \dot x_t = b_t(x_t),\qquad x_0\sim p_0
\end{equation}
for some drift $b:[0,1]\times\R^d\to\R^d$ chosen so that $\mathrm{Law}(x_1) = p_1 \equiv p_\mathrm{data}$, enabling sampling from $p_1$ by integrating the ODE. To this end, consider a \emph{stochastic interpolant}, which is a map $I$ defined by
\begin{equation}
I_t(x_0, x_1) = (1-\alpha_t) x_0 + \alpha_tx_1,
\end{equation}
satisfying $I_0(x_0, x_1) = x_0\sim p_0$ and $I_1(x_0, x_1)=x_1\sim p_1$ for some differentiable schedule $\alpha$. It defines a probability path between the two distributions of interest. By defining the drift as
\begin{equation}
    \forall 0 \leq t \leq 1,\qquad b_t(x) = \E\left[\dot I_t \mid I_t = x\right] = \E\Big[\dot \alpha_t x_1 - \dot\alpha_t x_0 \mid I_t = x\Big],
\end{equation}
we can ensure that $\mathrm{Law}(x_t) = \mathrm{Law}(I_t)$ at all times, and in particular at $t = 0$ and $t = 1$. Thus, learning to generate from the data distribution, $p_1$, reduces to learning the drift:
\begin{equation}
\label{eq:drift_regression}
\min_\theta \int_0^1 \E\left[\left\lVert v^\theta_t(I_t) - \dot I_t \right\rVert_2^2\right]\diff t,
\end{equation}
where $v^\theta$ is our modelled vector field, and the expectation is taken over $I_t$. We can also directly model an endpoint \emph{denoiser}, since $v_t(x) = \frac{\E[x_1\mid I_t = x] - x}{1-\alpha_t}$.

\subsection{Flow Map Matching}
Integrating the learnt vector field can be costly, requiring many function evaluations (NFEs) to obtain high quality samples. Instead, it is possible to learn a flow map, $X:[0,1]^2 \times \R^d \to \R^d$, which brings any solution of the probability flow from any time $s$ to another time $t$~\citep{flowmapmatching,howtobuildconsistencymodel,meanflows}. Formally, the flow map satisfies the following \emph{jump condition}:
\begin{equation}
\forall 0 \leq s, t \leq 1,\qquad X_{s,t}(x_s) = x_t.
\end{equation}
Reparameterising $X$ as
\begin{equation}
    X_{s,t}(x) = x + (t-s)v_{s,t}(x)
\end{equation}
for some two-time vector field $v:[0,1]^2\times \R^d\to\R^d$, we can characterise the flow map via three key identities, the presentation of which we defer to~\autoref{app:flowmaps}. Alongside those, it can be shown that
\begin{equation}
    \forall 0 \leq t \leq 1,\qquad v_{t,t}(x) = b_t(x),
\end{equation}
which relates the flow map to the drift of the probability flow. By enforcing simultaneously this tangent condition through~\eqref{eq:drift_regression} together with one of the consistency characterisations, we can learn the unique flow map of the ODE by modelling $v_{s,t}^\theta$. Sampling becomes a single pass: $X_{0,1}(x_0)\sim p_1$.

\paragraph{Categorical Flow Maps.} Introduced in~\citet{cfm,fmlm,dfm}, CFMs extend the flow map matching framework to the discrete case. Reparameterise the two-time vector field as
\begin{equation}
    v_{s,t}(x) = \frac{\pi_{s,t}(x) - x}{1-s}
\end{equation}
where $\pi:[0,1]^2\times\R^d\to\R^d$ is a ``partial'' denoiser. The tangent condition requires enforcing~\eqref{eq:drift_regression}, which reduces to a simple cross-entropy loss:
\begin{equation}
\label{eq:cross_ent_diag}
\mathcal{L}_\mathrm{diag}(\theta) \coloneqq -\E\left[\sum_{k=1}^{\lvert V\rvert}I_1^{(k)}\log \pi_{t,t}^{\theta,(k)}(I_t) \right],
\end{equation}
and the self-distillation objectives become analogous identities but on the endpoint prediction level~\citep{fmlm,dfm}. We consider in this work the empirically best-performing losses. The first one is the progressive self-distillation loss (PSD):
\begin{equation}
\label{eq:psd}
\mathcal{L}_{\mathrm{PSD}}(\theta) = \E\Big[\KL{\alpha_{s,u,t}\pi^\theta_{s,u}(I_s) + (1-\alpha_{s,u,t})\pi^\theta_{u,t}(X^\theta_{s,u}(I_s))}{\pi_{s,t}^\theta(I_s)}\Big],
\end{equation}
with $\alpha_{s,u,t} \in [0,1]$, which effectively amounts to a continuous-time shortcut model~\citep{shortcutmodels}. The second objective is the stabilised Eulerian-logit self-distillation (ESD)~\citep{dfm}:
\begin{equation}
\label{eq:esd}
\mathcal{L}_{\mathrm{ESD}}(\theta)=\E\Big[\KL{T^\theta_{s,t}(I_s)}{\pi_{s,t}^\theta(I_s)}\Big],
\end{equation}
where $T^\theta:[0,1]^2\times\R^d \to \R^d$ is the teacher distribution, given by
\begin{equation}
T^\theta_{s,t}(I_s) \coloneqq \mathrm{softmax}\Big( z_{s,s}^\theta(I_s) - \log\left((1-t)\mathbf{1} -(1-s)(t-s)\delta_{s,t}^\theta(I_s)\right) \Big),
\end{equation}
with $\delta_{s,t}^\theta(x) = D_s z_{s,t}^\theta(x) - \langle\pi_{s,t}^\theta(x), D_s z_{s,t}^\theta(x)\rangle\mathbf{1}$, and $z^\theta:[0,1]^2\times\R^d\to\R^d$ are the logits of $\pi^\theta$, for $\pi_{s,t}^\theta(x) = \mathrm{softmax}(z_{s,t}^\theta(x))$. While ESD produced the best numbers of~\citet{dfm}, our experiments at large scale have exhibited relatively unstable training. Due to that, in the main body of the paper, we focus on the PSD results. The ESD results are deferred to~\Cref{app:additional_esd_results}.

%% file: sections/method.tex
\section{Pretraining a flow}\label{sec:method}
\begin{wrapfigure}{r}{0.4\textwidth}
    \vspace{-2em}
    \centering
    \includegraphics[width=\linewidth]{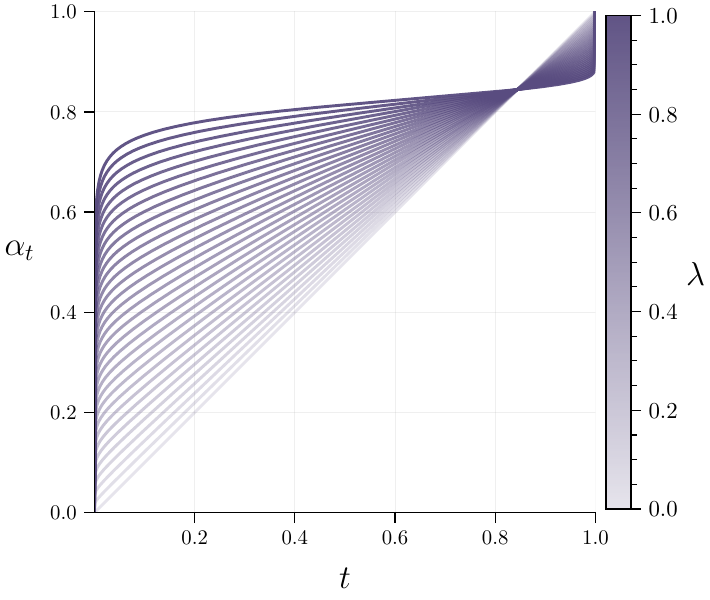}
    \caption{Mixture schedules for varying $\lambda \in [0,1]$ at vocabulary size $100$k.}
    \label{fig:mixture_schedules}
    \vspace{-2em}
\end{wrapfigure}
Current best-performing approaches for CFMs rely on a two-stage training~\citep{fmlm,dfm}: in the first stage, a simple cross-entropy denoiser is trained (see~\cref{eq:cross_ent_diag}); in the second stage, that denoiser is self-distilled into a CFM (see~\cref{eq:esd,eq:psd}). The two-stage approach is necessary, because the self-distillation objectives rely on a biased signal from the model itself, so applying them from the start degrades final performance. Moreover, self-distillation objectives are computationally demanding (requiring multiple forward passes per gradient step), making their optimisation over the entire training trajectory prohibitively expensive.

In this section, we investigate the importance of three components in the design space of CFM pretraining: the time schedule mixture, the adaptive loss weight, and the random ``unmasking'' of inputs. Finally, we sweep over optimiser hyperparameters, and discuss the possibility of their transfer to larger scales.

\subsection{Design space of CFMs}
\paragraph{Time scheduling.}\label{sec:bg_schedule}
\citet{fmlm,dfm} show the importance of the time schedule for CFMs. Instead of the linear interpolant schedule, a monotone map $\gamma:[0,1]\to[0,1]$, coined the ``error decoding'' schedule~\citep{duo,candi}, significantly improves the final model's performance. This schedule adapts the time so that the probability of incorrectly decoding the current point, $P(\argmax I_t \neq \argmax I_1)$, grows linearly in $t$, instead of exhibiting a sharp phase transition of the linear schedule. They report that a convex combination of it with the linear schedule, $\alpha_t = (1-\lambda)t + \lambda \gamma_t$, further improves the generation quality. While~\citet{dfm} train with $\lambda = 0.9$,~\citet{fmlm} use it only at inference with $\lambda \in \{0.5, 0.75, 1\}$. We ablate over this mixture choice.

\paragraph{Adaptive loss weighting.}\label{sec:bg_adaptive}
\citet{dfm} adopt the adaptive loss weight of~\citet{geng2024consistencymodelseasy}: for any objective of the form $\KL{p}{q}$, they compute $\Delta(x) = \lVert p(x) - q(x)\rVert_2^2$, and reweight by $w(x) = \stopgrad\left((\Delta(x) + c)^{-r}\right)$, with constants $c = 10^{-2}$ and $r = 1/2$ on the diagonal loss. The resulting objective is
\begin{equation}
    \mathcal L_{\mathrm{w}}(\theta) = \E_x\Big[w(x)\KL{p}{q}\Big].
\end{equation}
This upweights low-error regions and downweights high-error ones in order to stabilise training. We also assess the importance of this weight.

\paragraph{Random unmasking of input sequence}
The models are by default only trained on generating full, unconditional sequences, but final downstream applications typically require conditional generation given prompts. To match this setting, we follow~\citet{dieleman2022continuousdiffusioncategoricaldata}, and with a probability $p$ unmask a random prefix of the text, sampled from $\mathcal{U}[1,L/d]$, with $d \mid L$. We vary both $p$ and $d$.

\subsection{Hyperparameter sweep}
\begin{figure*}[t]
    \centering
    \includegraphics[width=\textwidth]{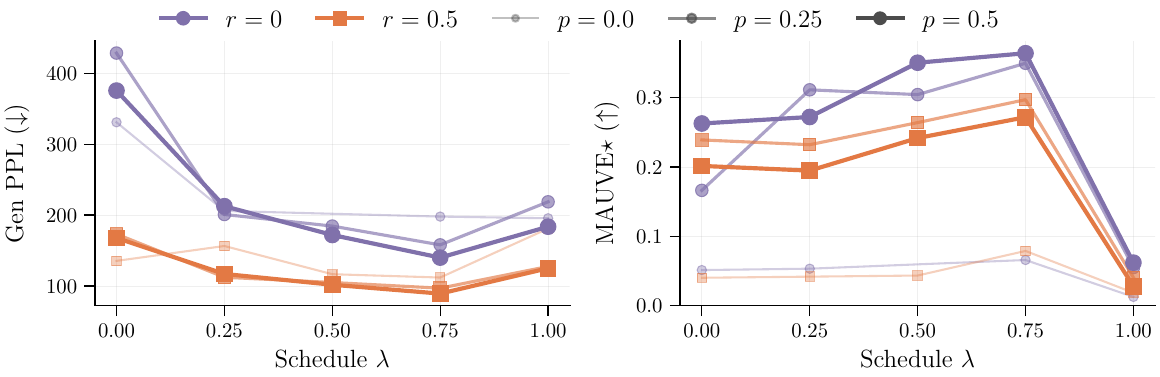}
    \caption{Ablation over schedule $\lambda$, adaptive loss reweighting $r$, and clean prefix probability $p$: Generative Perplexity and \mauve score as a function of $\lambda$.}
    \label{fig:ablation}
\end{figure*}
\paragraph{Setup.} We train a $400$M-parameter adaptation of diffusion transformers from SEDD~\citep{sedd}, which is also used in~\citet{cfm,fmlm}. In particular, we use the transformer from~\citet{cfm}, which follow the recommendations of~\citet{tvm} to upper-bound the Lipschitz constant of the model for more stable training. We train over $16$B tokens, resulting in about $40$ tokens per parameter over the Nemotron pre-training dataset~\citep{nvidia2025nvidianemotronnano2} using the TikToken tokenizer~\citep{tiktoken}. As in~\citet{LLaDa}, we also randomly crop sequences to ensure the models have been trained on different context lengths. We sweep over the mixture schedule $\lambda \in [0, 0.25, 0.5, 0.75, 1]$, adaptive loss weight on ($r = 0.5$) and off ($r = 0$), and the clean prefix probability $p \in [0, 0.25, 0.5]$.

\paragraph{Metrics.} We evaluate our models using \emph{generative perplexity} (Gen PPL, $\downarrow$) and \mauve ($\uparrow$)~\citep{MAUVE} over $128$ inference steps. For Gen PPL, we use Qwen 2.5 $7$B~\citep{qwen} to compute the perplexity of unconditionally generated samples from our models with $L = 1024$, measuring general sample quality. For \mauve, we extract sequences of length $L = 512$ from Nemotron, drop their second half, and use our model to generate the completion. Embeddings computed on both the original samples and the generated samples with GPT-2~\citep{gpt2} are used to measure similarity. A score of $1.0$ indicates completions indistinguishable from the original ones on the distributional level, measuring prompt completion quality and diversity.

\paragraph{Results.}
\Cref{fig:ablation} summarises our first sweep over a total of $30$ configurations. Adaptive loss weighting ($r = 0.5$) uniformly improves Gen PPL but slightly degrades \mauve, revealing a trade-off between token-level quality and distributional fidelity. Since \mauve is upper-bounded by $1$, the Gen PPL gain is more practically meaningful than the \mauve degradation at these values. We confirm this in~\Cref{sec:experiments}, where \mauve reaches near-maximal values (about $0.94$) despite having the adaptive loss weight enabled. The schedule optimum lies near $\lambda = 0.75$, with a sharp collapse at $\lambda = 1$, which shows that pure error-decoding without mixing is detrimental. Clean prefix probability $p > 0$ is essential for \mauve: without it, scores remain near zero regardless of other settings, resulting in out-of-distribution outputs. In a follow-up grid of $324$ configurations (\Cref{fig:marginals_006}), we narrow the ranges around these best settings and find that Gen PPL and \mauve are primarily sensitive to learning rate and weight decay, while $p$, $d$, and $\lambda$ are robust in the selected range.

\begin{figure*}[t]
    \centering
    \includegraphics[width=\textwidth]{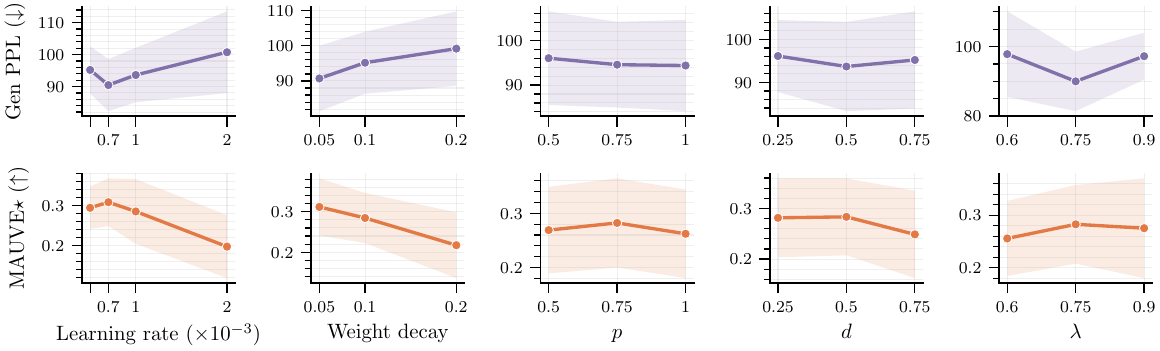}
    \caption{Marginal effect of each swept hyperparameter on Gen PPL (top) and \mauve (bottom), averaged over all other variables ($324$ configurations). Shaded regions denote $\pm 1$ standard deviation. Learning rate and weight decay are the primary drivers; $p$, $d$, and $\lambda$ are robust in the narrowed range.}
    \label{fig:marginals_006}
\end{figure*}

\subsection{Hyperparameter transfer}
Searching for optimal hyperparameters directly at large scale is expensive. Therefore, it is common to rely on zero-shot hyperparameter transfer rules to identify optimal hyperparameters when scaling the model size, batch size, and token horizon~\citep{tensorprograms,completeP,completedP}. We attempted to leverage Complete(d)P~\citep{completedP} by performing a hyperparameter sweep on a $400$M model; see~\Cref{app:sweep_details}. Since our architecture includes time-conditioning layers that are not covered by existing transfer theories, we adapted the method by treating these layers as hidden layers. However, our experiments in~\Cref{app:completedp} show that the baseline approach, with neither depth nor width transfer, achieved the best performance. We therefore leave zero-shot hyperparameter transfer for CFMs as an open question.

%% file: sections/experiments.tex
\section{Experiments}\label{sec:experiments}
For the full-scale pretraining experiment, we closely follow the setup of~\citet{scalingbeyondmaskeddiffusion}: we train a $1.7$B-parameter CFM on the Nemotron Pretraining dataset~\citep{nvidia2025nvidianemotronnano2} using the TikToken tokenizer~\citep{tiktoken} ($|V| \approx100\text{k}$), with a context length of $2048$ tokens. We train on $2.1$T tokens using $256$ H100s with a batch size of $10$ per GPU. We use the architecture of~\Cref{app:architecture}. For the self-distillation phase, we train the models for $200$k steps with $64$ H100s and a batch size of $6$, for both PSD and ESD losses.\footnote{For the self-distillation round, the number of optimisation steps is more relevant than the number of tokens.} For the ESD loss, note that we use finite differences to estimate the derivatives, as using Jacobian-Vector Products (JVPs) on such large models is prohibitively expensive in both memory and time. We find this approach computationally more efficient than the usage of custom kernels~\citep{tvm,cfm}. Additionally, we train an autoregressive baseline at the same scale for comparison.

\subsection{1.7B flow pre-training}
We present our results in~\Cref{fig:sd_flowmap_pareto,fig:mauve_star} as the ``Undistilled'' lines. At high inference steps, the flow model reaches Gen PPL competitive with the AR baseline: at $1024$ steps, our model attains $58.66$ Gen PPL, vs.\ $45.33$ for the AR model. As for the \mauve score, the model reaches near maximal values at high inference steps (approximately $0.94$), and, already at $8$ steps, exceeds the AR baseline: $0.88$ (ours) vs.\ $0.7$ (AR). In~\Cref{app:training_dynamics}, we provide additional details on the training dynamics of the model, where we can see that the generation quality steadily improves. Furthermore, we also provide qualitative samples from the undistilled model in~\Cref{app:qualitative_samples}.

\subsection{Self-distillation phase}
\subsubsection{Design space of self-distillation}
\begin{figure}[t]
    \centering
    \includegraphics[width=\textwidth]{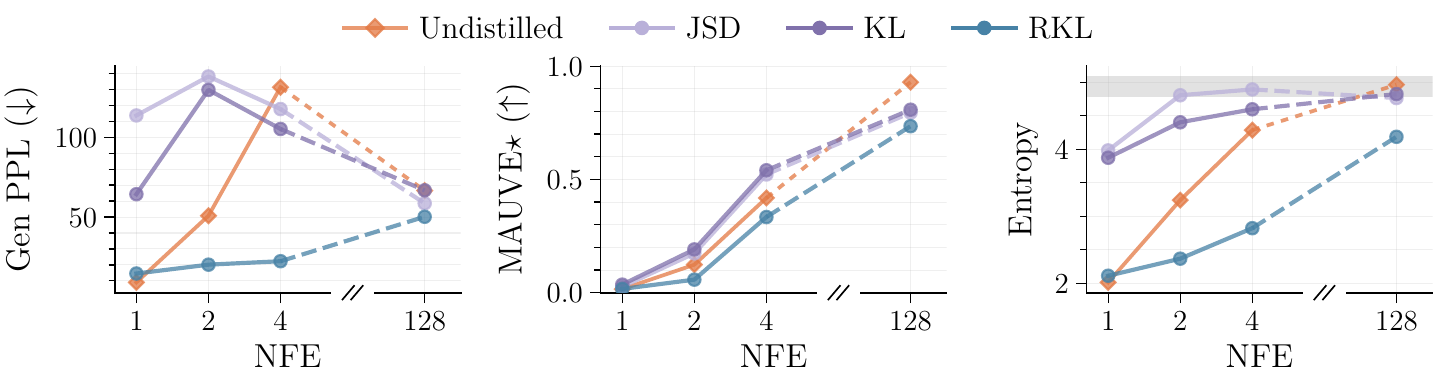}
    \caption{Ablation on PSD self-distillation, $10$k steps: marginal effect of the divergence on \mauve and Gen PPL, averaged over the time-sampling distribution.}
    \label{fig:psd_lines_kl}
\end{figure}
\begin{figure}[t]
    \centering
    \includegraphics[width=\textwidth]{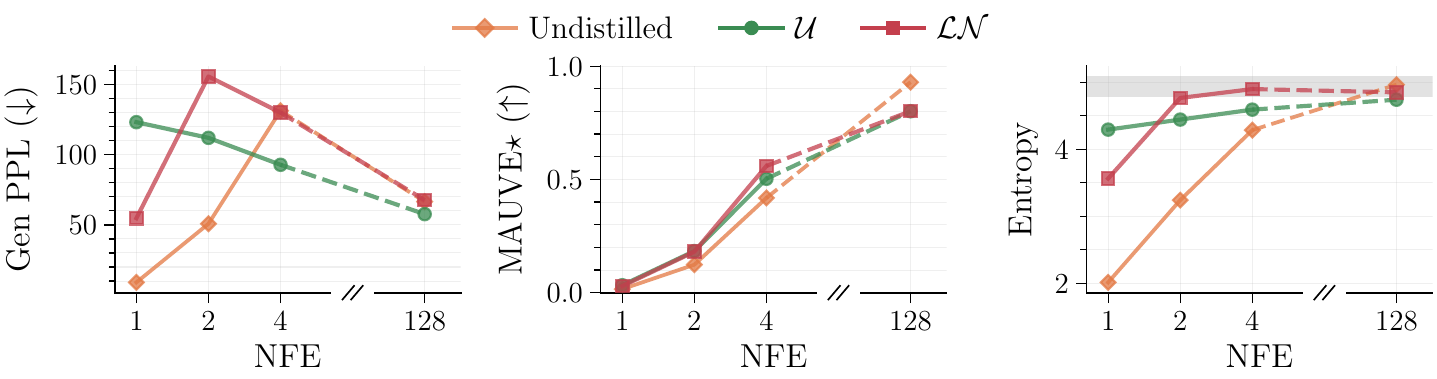}
    \caption{Ablation on PSD self-distillation, $10$k steps: marginal effect of the $(s,t)$ time-sampling distribution on \mauve and Gen PPL, averaged over the divergence without the reverse KL data.}
    \label{fig:psd_lines_time}
\end{figure}

We now investigate how to self-distill the pretrained flow into a CFM. The design space is large, and we explore two axes on short $10$k step runs. The first one is the joint time distribution of $(s,t)$: we consider the gap-schedule used in~\citet{fmlm}, starting from $10^{-2}$ and doubling every $500$ steps (denoted $\mathcal{U}$), and the gap logit-normal distribution of~\citet{tvm} (denoted $\mathcal{LN}$). As pointed out by~\citet{dfm}, we can choose any type of divergence $D$ for the SD objectives. This constitutes the second axis, as we try the following divergences: forward KL (the default choice, denoted $\mathrm{KL}$), reverse KL ($\mathrm{RKL}$) and the Jensen-Shannon Divergence ($\mathrm{JSD}$). The optimisation dynamics are discussed in~\Cref{app:divergence_choice}: the forward KL reinforces positions on which the teacher and the student concur; in contrast, JSD does not exhibit this behaviour. Avoiding this reinforcement could be preferable, as the teacher is biased and imperfect, and could thus alter the underlying flow.

We track Gen PPL and \mauve for $1$, $2$, $4$ and $128$ inference steps to observe the effect on few-step generation, and to ensure that the larger step regime remains stable. We report our results in~\Cref{fig:psd_lines_kl} and~\Cref{fig:psd_lines_time} for PSD. We omit ESD results as all runs diverged; for thoroughness, they are available in~\Cref{fig:esd_lines_kl,fig:esd_lines_time}. ESD at this scale exhibited either mode collapse or degraded \mauve relative to the undistilled flow, which we attribute to high-variance Jacobian estimates from finite differences. Therefore, we focus on PSD. All self-distillation objectives, except those based on RKL, improve the token entropy of single- and few-step generation, which is also reflected in higher \mauve at $4$ NFE. The uniform distribution produces better entropy at a single step, and better overall quality for NFE $\geq 2$. With the logit-normal distribution, single-step samples have lower entropy, which explains the lower Gen PPL. Based on these findings, we select PSD with both forward KL and JSD divergences, using the uniform time distribution, for the full $200$k-step distillation.

\subsubsection{Full self-distillation}
\begin{figure}[tb]
    \centering
    \begin{minipage}[t]{0.48\linewidth}
        \centering
        \includegraphics[width=\linewidth]{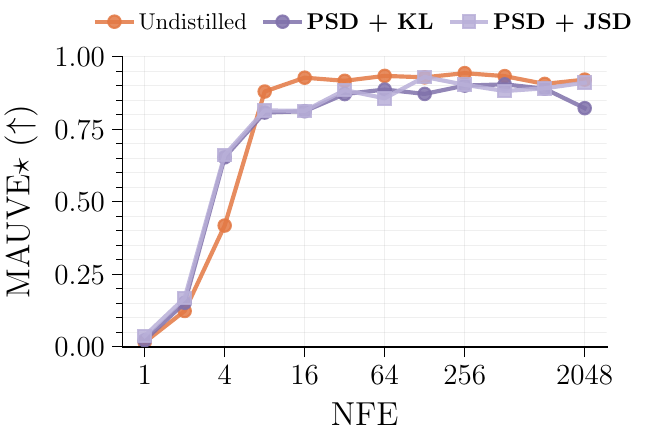}
        \caption{{\mauve vs.\ NFE.} A significant gap between the distilled and undistilled models in the low NFE regime is observable.}
        \label{fig:mauve_star}
    \end{minipage}
    \hfill
    \begin{minipage}[t]{0.48\linewidth}
        \centering
        \includegraphics[width=\linewidth]{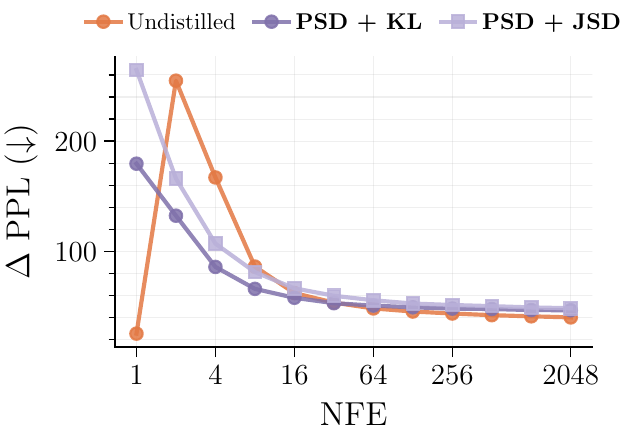}
        \caption{\deltappl (conditional generation) vs.\ NFE. The undistilled model produces collapsed outputs for $1$ NFE.}
        \label{fig:delta_ppl}
    \end{minipage}
\end{figure}

The results of the full-scale, $200$k steps self-distillation are available in~\Cref{fig:sd_flowmap_pareto} and~\Cref{fig:mauve_star}. Both flow-maps provide samples substantially closer to the data distribution in terms of token entropy. Moreover, while the undistilled model produced degenerate samples in the $1$ -- $2$ step regime, which artificially lowers the Gen PPL, the flow map models generate diverse samples, though Gen PPL remains high. The \mauve is improved by both models in the few-step regime, with the most prominent improvement at $4$ sampling steps. 
The JSD overall seems to slightly degrade generation quality; we hypothesise this is due to the higher token entropy it produces compared to forward KL.  
 
We additionally report \emph{Delta Perplexity} (\deltappl), which measures a model's ability to complete text conditioned on a dataset prompt. Given a prompt from the dataset, the evaluated model generates the remaining text, and we report the perplexity of the generated completion under Qwen~2.5~7B~\citep{qwen}. Depicted in~\Cref{fig:delta_ppl}, the \deltappl of the self-distilled models improves significantly in the few-step regime, while recovering similar perplexities for more inference steps.

\subsection{Likelihood-based benchmarks}\label{sec:elbo}
\paragraph{Semi-discrete ELBO.} Standard LM benchmarks require scoring answer choices by likelihood, which is not directly available from a flow-based model. To bridge this gap, we derive the following ELBO for the semi-discrete setting, extended from~\citet{vdm}. We denote by $\mathbf{e}_x$ the one-hot encoding of $x \in V$.

\begin{propbox}
\begin{restatable}[Semi-discrete ELBO]{proposition}{semidiscreteelbo}\label{prop:elbo}
Let $x \in V$ be a category. Let $\ell_{\mathrm{CE}}$ denote the cross-entropy loss. The following bound holds on the log likelihood of the model $p_\theta$:
\begin{equation}
    -\log p_\theta(x) \leq \E_{I_t}\left[\frac{2\dot\alpha_t\alpha_t}{(1-\alpha_t)^3} \ell_{\mathrm{CE}}\left(I_1, \pi_{t,t}^\theta(I_t)\right) \mid I_1=\mathbf{e}_x\right] + \ell_{\mathrm{CE}}(x, \pi_{1,1}^\theta(\mathbf{e}_x)).
\end{equation}
\end{restatable}
\end{propbox}
The proof is deferred to~\Cref{app:elbo_proof}. The semi-discrete ELBO spares us the reverse integration due to the instantaneous change of variables~\citep{ffjord,chenNODE,flowmatching}, which, in this setting, is particularly costly, as its evaluation additionally requires an integration over $\Delta^d_i \coloneqq \{p \in \Delta^d : \argmax p_j = i\}$ and as the model size is large. We discuss in more detail the issue of likelihood estimation for our setting in~\Cref{app:further_likelihoods}.

\paragraph{LM Harness.} We evaluate CFMs on standard multiple-choice question answering (MCQA) benchmarks using the ELBO estimator above.
We compare against an autoregressive baseline that we train under the same setup, and against the baselines reproduced by~\citet{scalingbeyondmaskeddiffusion} and prior work; see~\Cref{tab:likelihood_benchmarks}. Our ELBO yields non-random accuracies across most benchmarks, placing the CFM in the same range as Eso-LMs~\citep{esolms}, while remaining below discrete diffusion methods that are trained directly on ELBO maximisation. Since it is unclear how tight the bound is, refining the likelihood estimation remains a natural avenue for improvement.

\begin{table}[t]
\centering
\caption{Zero-shot accuracy (\%) on likelihood-based benchmarks. $^\dagger$Models from~\citet{scalingbeyondmaskeddiffusion} ($1.7$B params, same data and compute budget). $^*$External prior work. CFM likelihood is estimated via MC ELBO. Best non-AR results per column are in \textbf{bold}.}
\label{tab:likelihood_benchmarks}
\vskip 0.1in
\begin{small}
\definecolor{oursrow}{HTML}{E8EBF0}
\begin{tabularx}{\textwidth}{@{}lXcccccc@{}}
\toprule
& \textbf{Method} & \textbf{ARC-E} & \textbf{BoolQ} & \textbf{OBQA} & \textbf{PIQA} & \textbf{RACE} & \textbf{SIQA} \\
\midrule
& \textit{Chance} & $24.7$ & $50.4$ & $26.6$ & $51.6$ & $24.2$ & $32.2$ \\
\midrule
\multirow{6}{*}{\rotatebox[origin=c]{90}{\scriptsize Prior work}}
& SMDM-1B$^*$ & $37.4$ & $61.5$ & $27.0$ & $60.3$ & $29.3$ & $37.9$ \\
& LLaDa-8B-Base$^*$~\citep{LLaDa} & -- & -- & -- & $74.4$ & -- & -- \\
\cmidrule{2-8}
& AR$^\dagger$ {\small\textit{(autoregressive)}} & $72.7$ & $71.9$ & $40.4$ & $78.1$ & $36.2$ & $41.9$ \\
& AR {\small\textit{(autoregressive, ours)}} & $73.1$ & $63.6$ & $40.2$ & $76.9$ & $36.3$ &  $43.5$ \\
\cmidrule{2-8}
& Duo$^\dagger$ {\small\textit{(uniform diffusion)}} & $\mathbf{53.4}$ & $59.6$ & $\mathbf{33.0}$ & $\mathbf{62.7}$ & $\mathbf{35.0}$ & $39.0$ \\
& MDLM$^\dagger$ {\small\textit{(masked diffusion)}} & $50.5$ & $\mathbf{62.8}$ & $32.0$ & $62.2$ & $34.7$ & $\mathbf{39.2}$ \\
& Eso-LM$^\dagger$ {\small\textit{(interpolating)}} & $46.0$ & $53.4$ & $29.6$ & $55.6$ & $26.1$ & $36.1$ \\
\midrule
\rowcolor{oursrow} & CFM {\small\textit{(ours)}} & $47.1$ & $42.5$ & $24.4$ & $57.2$ & $28.3$ & $37.9$ \\
\bottomrule
\end{tabularx}
\end{small}
\end{table}

%% file: sections/related.tex
\section{Related work}\label{sec:related}
\paragraph{Continuous diffusion for discrete data.}
Early work by~\citet{dieleman2022continuousdiffusioncategoricaldata} attempted to learn the embeddings of the discrete data using the reparametrisation trick;~\citet{li2022diffusionlmimprovescontrollabletext} explored a similar approach with word embeddings. Later,~\citet{dirichletflowmatching} designed an interpolation scheme directly on the simplex.~\citet{fisherflowmatching,categoricalflowmatching} leveraged Riemannian geometry to alleviate the issues of linear interpolants.
Variational Flow Matching~\citep{vfm} shows how to flow from a Gaussian to one-hot encoded data. CFM~\citep{cfm}, FMLM~\citep{fmlm}, and DFM~\citep{dfm} adopted this approach, and reached competitive performance with discrete diffusion baselines. Our work builds directly on this line, and demonstrates its scalability to the $1.7$B parameter regime.

\paragraph{Accelerated sampling.}
Consistency models~\citep{consistencymodels,song2023improvedtechniquestrainingconsistency} pioneered the idea of mapping any point along a diffusion trajectory to the trajectory's endpoint, enabling one- or few-step generation.
Flow map matching~\citep{flowmapmatching,howtobuildconsistencymodel,meanflows} provides a principled framework for learning such maps via self-distillation of a pretrained flow. CFMs~\citep{cfm,fmlm,dfm} port flow map matching to the semi-discrete setting, reducing inference to a single forward pass while retaining the continuous-state benefits.

Other related works are discussed in~\Cref{app:related_work}.

%% file: sections/limitations.tex
\section{Limitations}
CFMs require materialising full $L \times \lvert V\rvert$ continuous-state trajectory matrices for all steps. At our scale ($L = 2048$, $\lvert V\rvert \approx 100\text{k}$), this is manageable given sufficient hardware ($256$ H100s), but the cost grows linearly in context length and vocabulary size. We absorb this bottleneck through compute and efficient cross-entropy implementation~\citep{DBLP:conf/iclr/WijmansHHKK25}, without solving it. Moreover, CFMs also present challenges in principled hyperparameter transfer. Unlike standard transformers, CFMs condition on continuous times via dedicated layers, which fall outside the parametrisations covered by Complete(d)P~\citep{completedP}. Our attempts to apply Complete(d)P to this architecture resulted in worse performance than direct tuning at the target scale. Adapting time-conditioned variants of $\mu$P~\citep{zheng2025scaling} is left for future work.

%% file: sections/conclusion.tex
\section{Conclusion}
We have shown that Categorical Flow Maps scale competitively to the $1.7$B-parameter regime. After self-distillation, they produce diverse, high-quality text in as few as $4$ function evaluations while maintaining near-data-level token entropy, pushing the quality vs.\ throughput Pareto front beyond discrete diffusion baselines. Our semi-discrete ELBO places the model in the same ballpark as other non-autoregressive methods, though a gap remains with approaches trained directly on ELBO maximisation. These results suggest that continuous-state generation is a viable path for discrete data at LLM scale, and that the broader flow map toolbox (guidance, tilting, reward fine-tuning) is now within reach for language modelling.

%% file: sections/appendix.tex
\section{Related work}\label{app:related_work}
\paragraph{Continuous diffusion for discrete data.} With a different approach,~\citet{langflow} proposed a way of learning continuous trajectories for discrete data, leveraging embedding spaces and Bregman divergences.

\paragraph{Discrete diffusion.}
Diffusion models for discrete state spaces originate from D3PM~\citep{d3pm} and multinomial diffusion~\citep{argmaxflows}, which design discrete forward processes whose reversal can be learned by a neural network.
\citet{campbell2022ctmc} recast the framework in continuous time and consider the rate transition matrices of the reverse CTMC.
SEDD~\citep{sedd} replaces the standard ELBO with a score-entropy objective, improving both training stability and sample quality.
Masked diffusion~\citep{mdlm,shi2024maskeddiffusion} simplified the earlier approaches, resulting in a simple cross-entropy objective that forms an ELBO on the data distribution. Discrete Flow Matching~\citep{discreteflowmatching} generalised the framework by learning a probability velocity field over discrete states.~\citep{schiffUSDM} further investigated the uniform-state case of discrete diffusion.

\paragraph{Accelerated sampling for discrete diffusion.} Several approaches have attempted to implement accelerated methods for discrete diffusion.~\citet{duo} shows how a continuous Gaussian latent can generate a uniform-state discrete diffusion process, enabling consistent trajectories for SDTT-like self-distillation~\citep{sdtt}.~\citet{di4c} show that by incorporating correlation through diffusion steps the few-step sampling regime can be improved. Based on similar observations,~\citet{redi} propose to refine the couplings between source and target distributions.

\section{Extended presentation of flow maps}\label{app:flowmaps}
\subsection{Formal presentation}
As described in the background section, flow maps satisfy the tangent condition and are characterised by three additional equivalent properties. We first begin by proving the tangent condition.

\definecolor{apurp}{HTML}{A196CC}
\definecolor{apurpd}{HTML}{594C80}
\definecolor{aorange}{HTML}{FA975C}
\definecolor{aoranged}{HTML}{E37944}
\definecolor{agray}{HTML}{74808B}
\definecolor{agraylt}{HTML}{C7D1D6}

\begin{lemma}[Tangent condition]
For any solution of the probability flow, $(x_t)_{t\in[0,1]}$, 
\begin{equation}
    \forall 0 \leq t \leq 1,\qquad v_{t,t}(x_t) = b_t(x_t).
\end{equation}
\end{lemma}
\begin{proof}
Consider the partial derivative with respect to $t$ of the flow map applied at a solution of the probability flow:
\begin{equation}
    \partial_t X_{s,t}(x_s) = v_{s,t}(x_s) + (t-s)\partial_t v_{s,t}(x_s).
\end{equation}
By definition, we also know that
\begin{equation}
    \partial_t X_{s,t}(x_s) = \partial_t x_t = b_t(x_t).
\end{equation}
Combining the two relations and taking the limit as $s$ tends to $t$, we arrive at the desired conclusion.
\end{proof}
The tangent condition is not surprising if we consider the two-time vector field to be the integral of the drift of the flow, it can also be derived by using Leibniz's integral rule.

We can now present the three characterisations from~\citet{flowmapmatching,howtobuildconsistencymodel}.

\begin{proposition}
A flow map of the probability flow satisfies the following three conditions, for any $0 \leq s,u,t, \leq 1$ and any solution of the probability flow:
\begin{itemize}
    \item \textbf{Lagrangian condition}. $\partial_t X_{s,t}(x_s) = v_{t,t}(X_{s,t}(x_s))$.
    \item \textbf{Semigroup condition}. $X_{u, t} \circ X_{s,u} = X_{s,t}$.
    \item \textbf{Eulerian condition}. $\partial_s X_{s,t}(x_s) + \nabla X_{s,t}(x_s) v_{s,s}(x_s) = 0$.
\end{itemize}
\end{proposition}

\begin{proof}
For the Lagrangian condition, it suffices to recall that $X_{s,t}(x_s) = x_t$, which implies, by taking the partial derivative with respect to time, that $\partial_t X_{s,t}(x_s) = v_{t,t}(x_t) = v_{t,t}(X_{s,t}(x_s))$.

The semigroup condition is trivial.

The Eulerian condition comes from the fact that $X_{s,t}\circ X_{t,s} = \mathrm{Id}$, of which one can take the total derivative with respect to $s$, yielding the required identity.
\end{proof}

Each of these conditions yields a simple MSE objective. When null, and when the flow matching objective is null too, the learnt map is the unique flow map of the probability flow. Formally, the flow map learning objective amounts to
\begin{equation}
    \mathcal L(\theta) = \E \left\lVert v^\theta_{t,t}(I_t) - \dot I_t\right\rVert^2 + \mathcal L_{\mathrm{SD}} (\theta),
\end{equation}
where, in practice, $\mathcal L_{\mathrm{SD}}(\theta)$ is either one of
\begin{itemize}
    \item \textbf{Lagrangian self-distillation}. $\mathcal L_{\mathrm{LSD}}(\theta) = \E\left\lVert \partial_t X^\theta_{s,t}(I_s) - \stopgrad(v^\theta_{t,t}(X^\theta_{s,t}(I_s)) \right\rVert^2$
    \item \textbf{Progressive self-distillation}. $\mathcal{L}_{\mathrm{PSD}}(\theta) = \E\left\lVert X^\theta_{s,t}(I_s) - \stopgrad(X^\theta_{u,t}(X^\theta_{s,u}(I_s)) \right\rVert^2$ 
    \item \textbf{Eulerian self-distillation}. $\mathcal{L}_\mathrm{ESD}(\theta) = \E\left\lVert \partial_s X_{s,t}(x_s) - \stopgrad(\nabla X_{s,t}(I_s)v_{s,s}(I_s)) \right\rVert^2$,
\end{itemize}
with $\stopgrad$ denoting the stop-gradient operator, which is such that $\nabla_\theta \stopgrad(f(\theta)) = \mathbf 0$. Note that the placement of $\stopgrad$ does not change the minimiser of those objectives, and is put in those positions in practice for stability and memory usage reasons.

\subsection{Algorithms}\label{app:algorithms}
We provide the sampling algorithm for a denoiser (flow matching) or a two-time denoiser (flow map) in~\Cref{alg:vf_sample}.

\begin{figure}[t]
\begin{algorithm}[H]
\caption{Sampling: endpoint flow matching or flow map}\label{alg:vf_sample}
\begin{algorithmic}[1]
\Require Trained denoiser $\pi^\theta_t$ or two-time denoiser $\pi_{s,t}^\theta$, schedule $0=t_0<\cdots<t_N=1$
\State Sample $x \sim p_0$
\For{$i = 0, \ldots, N-1$}
    \State $x \gets x + \frac{t_{i+1} - t_i}{1 - t_i}\!\left(\pi^\theta_{t_i}(x) - x\right)$\quad or $x \gets x + \frac{t_{i+1} - t_i}{1 - t_i}\!\left(\pi^\theta_{t_i, t_{i+1}}(x) - x\right)$
\EndFor
\State \Return $\arg\max(x)$
\end{algorithmic}
\end{algorithm}

\end{figure}

\section{Semi-discrete likelihoods}
\label{app:further_likelihoods}
In this section, we detail our attempts at extracting likelihoods out of a CFM. We note that we are interested in the following quantity, for a category $x \in V$,
\begin{equation}
\label{eq:semi_discrete_likelihood}
    p_\theta(x) = \int p_\theta(x\mid x_1)  p_\theta(x_1)\diff x_1.
\end{equation}
With argmax decoding, this reduces to
\begin{equation}
    p_\theta(x) = \int \mathbbm{1}\left(\argmax_j x_1^j = x\right) p_\theta(x_1) \diff x_1.
\end{equation}

\subsection{Reversing the ODE}
The continuous likelihood of the endpoint, $p_\theta(x_1)$, is accessible by reversing the ODE~\citep{ffjord,chenNODE}, as mentioned in~\Cref{sec:elbo}. Specifically, with the instantaneous change of variables, we have that
\begin{equation}
\log p_\theta(x_1) = \log p_\theta(x_0) - \int_{0}^1 \mathrm{div}(v_\theta(x_t)) \diff t,
\end{equation}
where, in our case,
\begin{equation}
\forall x \in \R^d, 0 \leq t < 1, \qquad v^\theta_t(x) = \frac{\pi^\theta_{t,t} (x) - x}{1-t}.
\end{equation}
The divergence can be efficiently estimated using the Hutchinson trace estimator~\citep{hutchinsonTrace}:
\begin{equation}
\mathrm{Tr}(A) = \E[z^\top A z],\qquad\text{ for } z\text{, s.t. }\E[zz^\top] = I.
\end{equation}
This avoids the materialisation of the Jacobian matrix through a Jacobian-Vector Product (JVP). While this is true, this computation remains prohibitively expensive in high dimensions. In our case, we operate over $\R^{L\times \lvert V\rvert}$ requiring many samples for the convergences of the above estimator, and applying a JVP to a $1.7$B model is also expensive. Besides, this computation has to be carried through time and space in our case, as described in~\eqref{eq:semi_discrete_likelihood}. We therefore cannot leverage this technique.

\subsection{Semi-discrete ELBO}\label{app:elbo_proof}
We prove, here, the semi-discrete ELBO stated in~\Cref{sec:elbo}.
\begin{propbox}
\semidiscreteelbo*
\end{propbox}
\begin{proof}
\citet{vdm} provide the following continuous-time ELBO:
\begin{equation}
- \log p_\theta(x) \leq \KL{q(I_0 \mid I_1=\mathbf{e}_x)}{p(I_0)} + \E_{q(z_0 \mid I_1=\mathbf e_x)}\left[-\log p_\theta(x\mid I_1 = \mathbf e_x)\right] + \mathcal{L}_\infty(x).
\end{equation}
The first term is the prior loss and is $0$ in our case. The second term is a cross entropy loss between the clean point, and our prediction at time $1$, assuming a continuous relaxation of the decoding of our model, replacing the argmax by a sample from $\mathrm{Cat}(x_1)$, where $x_1$ is the endpoint of our trajectory. The last term is the continuous-time diffusion loss and is given by
\begin{align}
\mathcal{L}_\infty(x) = \frac{1}{2}\E\int_{0}^1 \mathrm{SNR}'(t) \left\lVert \mathbf{e}_x - \pi_{t,t}^\theta(I_t)\right\rVert_2^2\diff t,
\end{align}
where $\mathrm{SNR}(t) = \frac{\alpha_t^2}{(1-\alpha_t)^2}$, for any $t$, the derivative of which corresponds to twice the weight in the final integral. Moreover, we have, by Pinsker's inequality, since $\mathbf{e}_x, \pi_{t,t}^\theta(I_t) \in \Delta^d$,
\begin{equation}
\lVert \mathbf{e}_x - \pi_{t,t}^\theta(I_t) \rVert_2^2 \leq 2\KL{\mathbf{e}_x}{\pi_{t,t}^\theta(I_t)} = 2\ell_\mathrm{CE}(x, \pi_{t,t}^\theta(I_t))
\end{equation}
\end{proof}
The tightness of this bound in the semi-discrete setting remains, however, unknown. We hypothesise that it is possible to extract better likelihoods out of the model.

\section{Experimental details}\label{app:experimental_details}

\subsection{Model configurations}

\Cref{tab:model_configs} summarises the model configurations.

\begin{table}[h]
    \caption{Model configurations.}
    \label{tab:model_configs}
    \centering
    \begin{tabular}{lcc}
        \toprule
        & \textbf{400M (sweep)} & \textbf{1.7B (scaling)} \\
        \midrule
        Hidden dimension $d$ & 1280 & 2304 \\
        Transformer blocks & 8 & 18 \\
        Attention heads & 20 & 36 \\
        Head dimension & 64 & 64 \\
        MLP expansion & 3$\times$ & 3$\times$ \\
        Context length & 1024 & 2048 \\
        Global Batch Size & 324 & 2560 \\
        \midrule
        Tokenizer & \multicolumn{2}{c}{TikToken} \\
        Vocabulary size & \multicolumn{2}{c}{$\sim\! 100$k} \\
        Training data & \multicolumn{2}{c}{Nemotron Pretraining dataset} \\
        Prior & \multicolumn{2}{c}{$\mathcal{N}(0, I)$} \\
        Precision & \multicolumn{2}{c}{bfloat16} \\
        Hardware & $8 \times$ H100 & $256 \times$ H100 \\
        \bottomrule
    \end{tabular}
\end{table}

\subsection{Architecture}\label{app:architecture}

The backbone is a transformer with AdaLN-Zero modulation~\citep{dit}, rotary position embeddings~\citep{RoPE}, SwiGLU MLP, and RMSNorm, following the Lipschitz recommendations of~\citet{tvm}. The network takes as input the noisy simplex representation $x \in \mathbb{R}^{L \times K}$ (sequence length $L$, vocabulary size $K$), together with two scalar time conditionings $(s, t)$ embedded via sinusoidal encodings and injected through AdaLN-Zero. For $s = t$, the model is a standard denoiser; for $s \neq t$, the two conditionings allow learning the full two-time flow map.

\subsection{Hyperparameter sweep details}\label{app:sweep_details}

The hyperparameter search was conducted in two stages, both at the $400$M proxy scale (\Cref{tab:model_configs}), trained for 100k steps on the Nemotron corpus.

\paragraph{Optimizer fundamentals (702 configurations).} We first sweep over learning rate, $\beta_1$, $\beta_2$, weight decay, $\varepsilon$, and dropout. \Cref{fig:optimizer_fundamentals} shows the marginal effect of each variable on CE loss. Learning rate has a clear optimum around $7 \times 10^{-4}$; higher weight decay improves CE loss; $\beta_2 = 0.95$ is consistently preferred over larger values, in line with the findings of~\citet{tvm}; $\beta_1$, $\varepsilon$, and dropout have negligible marginal effects ($\varepsilon$ insensitive, dropout weakly detrimental).

\begin{figure*}[t]
    \centering
    \includegraphics[width=\textwidth]{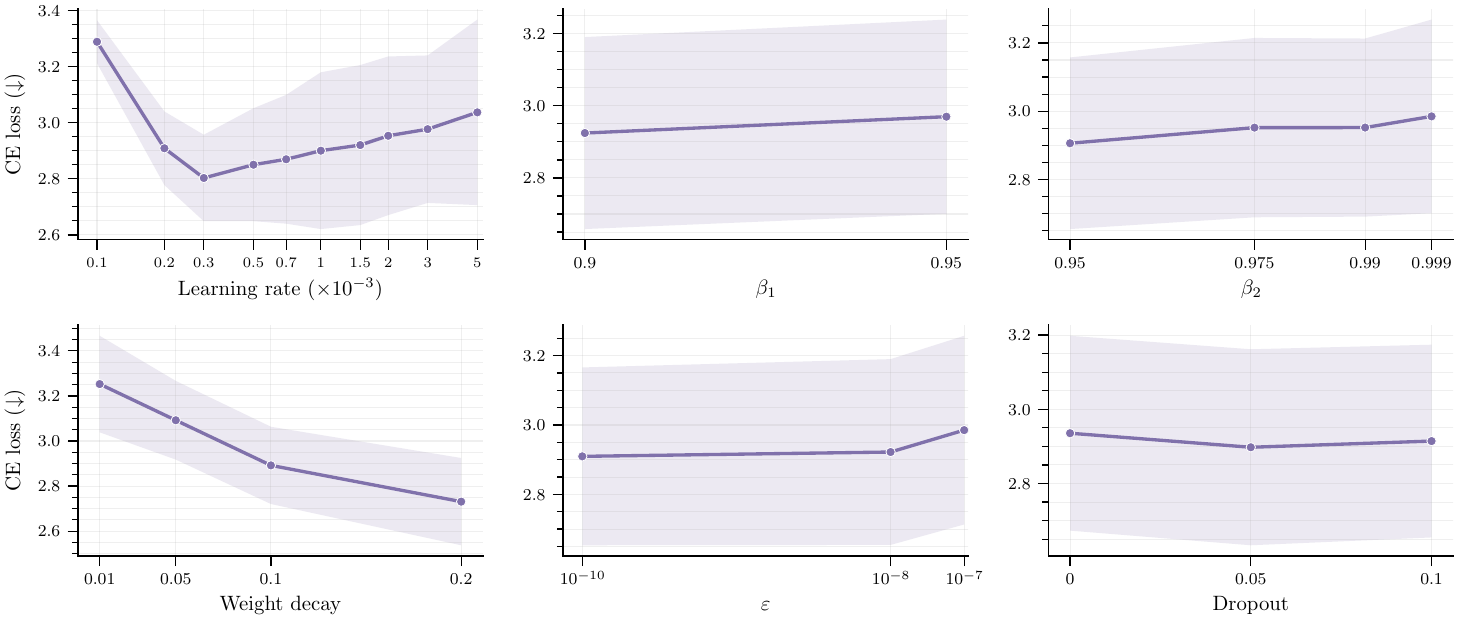}
    \caption{Marginal effect of optimizer hyperparameters on CE loss ($\pm 1$ std), averaged over $702$ configurations. $\beta_2 = 0.95$ is consistently preferred; weight decay helps; $\varepsilon$ and $\beta_1$ are insensitive.}
    \label{fig:optimizer_fundamentals}
\end{figure*}

\paragraph{Design space sweep (30 configurations).} We perform a grid search over the mixture schedule $\lambda \in \{0, 0.25, 0.5, 0.75, 1\}$, the adaptive loss reweighting exponent $r \in \{0, 0.5\}$, and the clean prefix probability $p \in \{0, 0.25, 0.5\}$, with all other hyperparameters fixed. This sweep identifies $\lambda \approx 0.75$, $r = 0.5$, and $p > 0$ as the best operating point (\Cref{fig:ablation}).~\Cref{fig:prefix_ablation} shows the complementary view with $p$ on the $x$-axis.

\paragraph{Optimizer sweep (324 configurations).} Fixing $r = 0.5$ and narrowing the ranges around the best design-space settings, we sweep over learning rate $\in \{5 \times 10^{-4},\, 7 \times 10^{-4},\, 10^{-3},\, 2 \times 10^{-3}\}$, weight decay $\in \{0.05, 0.1, 0.2\}$, $p \in \{0.5, 0.75, 1.0\}$, max prefix fraction $d \in \{0.25, 0.5, 0.75\}$, and $\lambda \in \{0.6, 0.75, 0.9\}$. The marginal effects (\Cref{fig:marginals_006}) show that learning rate and weight decay are the primary drivers, while $p$, $d$, and $\lambda$ are robust in the narrowed range. \Cref{fig:lr_wd} details the learning rate--weight decay interaction, and \Cref{fig:scatter} shows the Gen PPL vs \mauve trade-off across all configurations ($\rho = -0.51$).

\begin{figure*}[t]
    \centering
    \includegraphics[width=\textwidth]{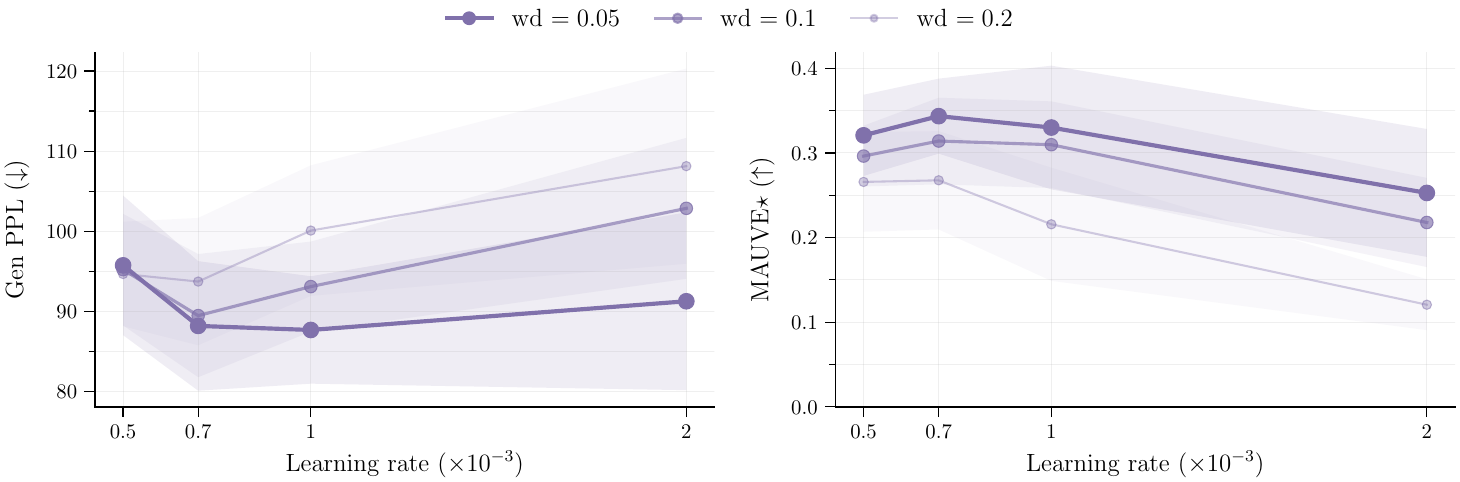}
    \caption{Gen PPL and \mauve as a function of learning rate, conditioned on weight decay ($\pm 1$ std shaded). Low weight decay consistently improves both metrics.}
    \label{fig:lr_wd}
\end{figure*}

\begin{figure*}[t]
    \centering
    \includegraphics[width=\textwidth]{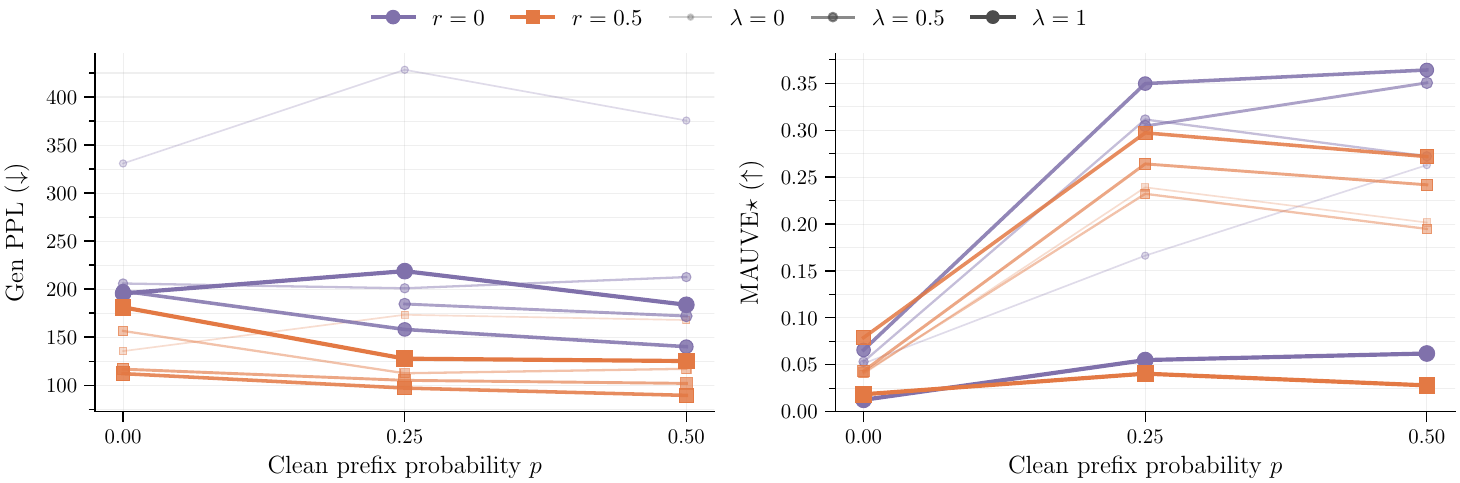}
    \caption{Gen PPL and \mauve as a function of clean prefix probability $p$. Colour encodes reweighting $r$, opacity encodes schedule $\lambda$. Increasing $p$ substantially improves \mauve with minimal impact on Gen PPL.}
    \label{fig:prefix_ablation}
\end{figure*}

\Cref{fig:scatter} shows the joint distribution of Gen PPL and \mauve across all $324$ configurations. The two metrics are moderately anti-correlated ($\rho = -0.51$), suggesting a trade-off between token-level quality and distributional fidelity for small training horizons.

\begin{figure}[t]
    \centering
    \includegraphics[width=0.7\textwidth]{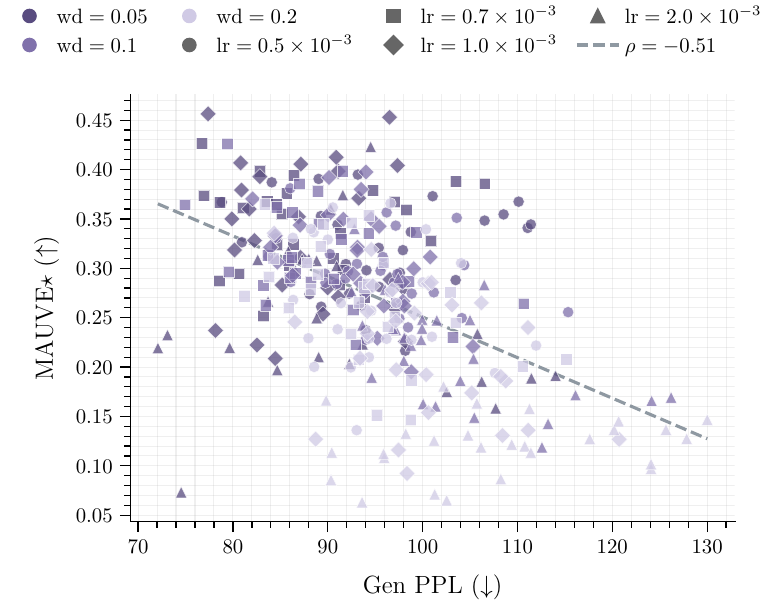}
    \caption{Gen PPL vs \mauve for all $324$ optimizer sweep configurations. Colour encodes weight decay, marker shape encodes learning rate. Dashed line: linear fit ($\rho = -0.51$).}
    \label{fig:scatter}
\end{figure}

\subsection{Complete(d)P hyperparameter transfer}\label{app:completedp}

We attempted to leverage Complete(d)P~\citep{completedP} to transfer hyperparameters from the 400M proxy scale to the 1.7B target scale. \Cref{fig:completedp} compares training loss curves with and without Complete(d)P transfer at batch size $1{,}280$ ($128 \times$ H100). The run without Complete(d)P consistently achieves lower CE loss throughout training. We find that Complete(d)P degrades Gen PPL by $35\%$ and \mauve by $10\%$.

\begin{figure}[t]
    \centering
    \includegraphics[width=0.7\textwidth]{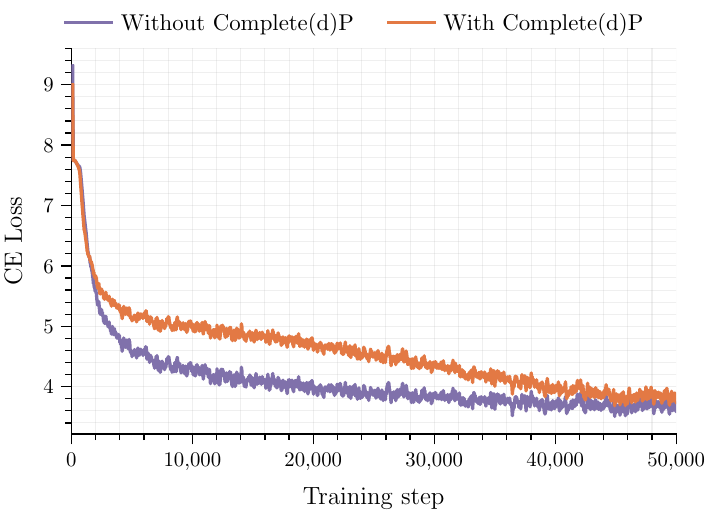}
    \caption{CE training loss with and without Complete(d)P hyperparameter transfer (batch size $1{,}280$, $128\times$H100). The transferred hyperparameters result in consistently higher loss.}
    \label{fig:completedp}
\end{figure}

\subsection{Training dynamics}\label{app:training_dynamics}
We provide in~\Cref{fig:training_dynamics} the evolution of \deltappl and \mauve through time using $128$ inference steps. Both improve rather steadily across time, but, given their relatively high variance, some parts do not follow the overall monotonic trend. We also report that the Gen PPL and entropy converge relatively fast to the right values, and remain stable across training.
\begin{figure*}[t]
    \centering
    \includegraphics[width=\textwidth]{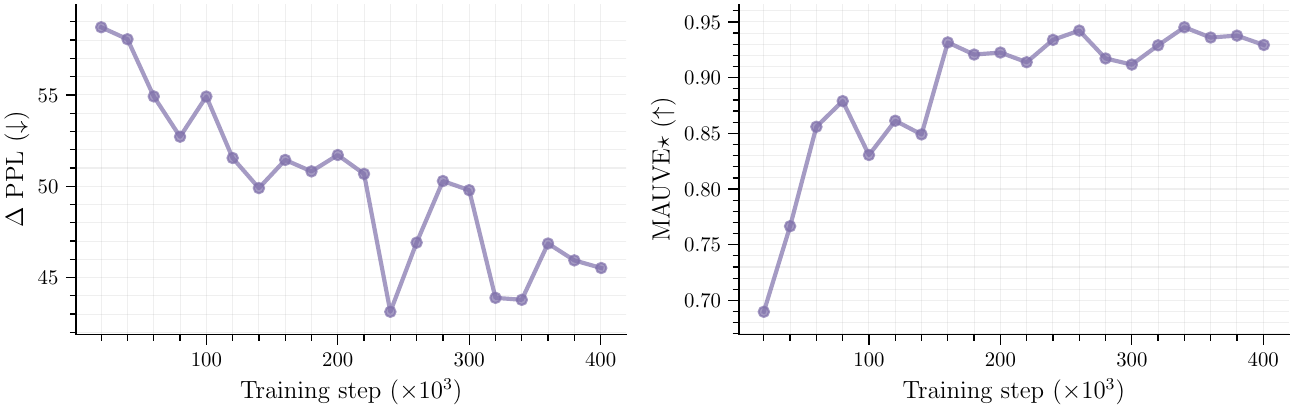}
    \caption{Generation quality metrics evaluated every 20k training steps using 128 Euler steps. (Left) \deltappl (evaluated with Qwen 2.5 7B) decreases steadily, indicating improving agreement with the autoregressive reference. (Right) \mauve rapidly saturates to ${\sim}0.93$ by 160k steps, showing that the model's unconditional distribution converges early in training.}
    \label{fig:training_dynamics}
\end{figure*}

\subsection{Self-distillation divergence ablation}\label{app:sd_ablation}
\subsubsection{Choice of divergence}
\label{app:divergence_choice}
In this work, additionally to the typically used forward KL divergence, we also attempted using reverse KL and Jensen-Shannon Divergence. As a reminder, for a model $p_\theta$ and a target $q$, the forward KL is
\begin{equation}
\label{eq:forward_kl}
\KL{q}{p_\theta} \coloneqq \E_q \log q - \E_q \log p_\theta,
\end{equation}
which results in the following gradients with respect to the parameters, $\theta$:
\begin{equation}
\nabla_\theta \KL{p_\theta}{q} = -\nabla_\theta\E_q\log p_\theta = -\E_q \frac{\nabla_\theta p_\theta}{p_\theta},
\end{equation}
recovering typical cross entropy training. Since $q_\theta$ weighs the loss term, it is highest when $p_\theta$ puts almost no density on a point where $q$ does. On the other hand, if $q$ puts no mass on a given point but $p_\theta$ does, the loss is zero. It is said to be ``mode-covering''. While this seems favourable, let us rewrite the gradient in clearer way:
\begin{equation}
    \nabla_\theta \KL{p_\theta}{q} = -\int \frac{q(x)}{p_\theta(x)} \nabla_\theta p_\theta(x) \diff x.
\end{equation}
Notice that when $q \approx p$, then the gradient under the integral is weighed by approximately $1$: this means that in positions where the model and the target agree already, the model is reinforced in that direction. While this may be favourable on the diagonal part of the loss, this might not hold for self-distillation part: the imperfect, biased teacher might further push the student towards the wrong distribution. It might be favourable that the self-distillation loss minimally alters the base model, to avoid biasing the model while learning the flow map, as we show with JSD.

The reverse KL, on the other hand, is defined as
\begin{equation}
\KL{p_\theta}{q} \coloneqq \E_{p_\theta} \log \frac{p_\theta}{q} = -H(p_\theta) - \E_{p_\theta}\log q,
\end{equation}
where $H(p_\theta)$ is the entropy of $p_\theta$. Minimising this objective results in a maximisation of the entropy of the model distribution and a minimisation of the reverse cross-entropy. In this objective, the $\log q$ terms prevails in terms of magnitude (remember that $q(x) \in [0,1]$, in the discrete case), and therefore, on any density point where $p_\theta$ and $q$ differ, the model is drawn back strongly towards it. For example, if $p_\theta$ puts some non-zero mass on a point $x$ and $q(x) \approx 0$, then the magnitude of the term may diverge. It is therefore said to be ``mode-seeking''. (We also note that the entropy is maximised when it is uniform.) The reverse KL presented the most instabilities as shown in the results.

Finally, the JSD is defined as follows. Let $m_\theta \coloneqq \frac{1}{2}\left(p_\theta + q\right)$ be the midpoint distribution between the model and the target. Then,
\begin{equation}
\JSD{p_\theta}{q} \coloneqq \frac{1}{2}\KL{p_\theta}{m_\theta} + \frac{1}{2}\KL{q}{m_\theta} = H(m_\theta) - \frac{1}{2}\Big(H(q) + H(p_\theta)\Big).
\end{equation}
It is minimising the entropy of the middle point while maximising the entropy of the model. As such, we can already note that it has the property that there may no be points such that $m^\theta$ and $p^\theta$ diverge completely on mass, by the middle-point's definition, therefore acting in some sense as a higher temperature term. Let us inspect its gradient:
\begin{align}
    \nabla_\theta H(m_\theta) &= -\frac{1}{2}\int \nabla_\theta\left[(p_\theta(x) + q(x))\log\frac{p_\theta(x) + q(x)}{2}\right]\diff x\\
    &= -\frac{1}{2}\int\left[ \nabla_\theta p_\theta(x)\log\frac{p_\theta(x) + q(x)}{2} +\big(p_\theta(x) + q(x)\big) \frac{\nabla_\theta p_\theta(x)}{p_\theta(x) + q(x)}\right] \diff x\\
    \nabla_\theta H(m_\theta) &= -\frac{1}{2}\int\left[ \nabla_\theta p_\theta(x)\log\frac{p_\theta(x) + q(x)}{2} +\nabla_\theta p_\theta(x)\right] \diff x.
\end{align}
The $p_\theta$ entropy term's gradient amounts to:
\begin{align}
\frac 12\nabla_\theta  H(p_\theta) &= -\frac 12 \int\left[\log p_\theta(x)\nabla_\theta p_\theta(x) + p_\theta(x) \frac{\nabla_\theta p_\theta(x)}{p_\theta(x)}  \right]\diff x \\
&= - \frac 12 \int \left[\log p_\theta(x) \nabla_\theta p_\theta(x) + \nabla_\theta p_\theta(x)\right].
\end{align}
Summing the two terms yields
\begin{align}
    \nabla_\theta \JSD{p_\theta}{q} = \nabla_\theta H(m_\theta) - \frac 12 \nabla_\theta H(p_\theta) = -\frac 12 \int \nabla_\theta p_\theta(x) \log \frac{p_\theta(x) + q(x)}{2p_\theta(x)} \diff x.
\end{align}
Otherwise stated:
\begin{equation}
    \nabla_\theta \JSD{p_\theta}{q} =  \frac 12 \int \nabla_\theta p_\theta(x) \log \frac{2p_\theta(x)}{p_\theta(x) + q(x)} \diff x.
\end{equation}
When the two distributions concur, the gradient is near zero. When $p_\theta$ tends to infinity, the $\log$ term tends to $\log 2$, therefore bounding gradients in the case of a (large) overestimation by the model. Remark as well that this term is always bounded by $\log 2$. When $p_\theta$ is small (tends to zero) where $q$ assigns mass, the term diverges to infinity. We see that the JSD alters the model's distributions in less aggressive ways than the forward KL divergence, and show empirically that the base flow is better preserved. Finally, the boundedness of the term, beyond avoiding divergence, naturally balances out the training dynamics between the diagonal $s=t$ loss and the self-distillation---an issue in the case of the KL divergences, where the diagonal term is of roughly the same order of magnitude of the self-distillation one, leading to the possible and detrimental domination of the latter.\footnote{If the flow term remains high, but the self-distillation remains close to zero, then the model would have effectively learnt a trivial solution of the self-distillation loss landscape.}

\subsubsection{Additional results for ESD}
\label{app:additional_esd_results}
As mentioned in~\Cref{sec:experiments}, we present our ablation results on the ESD-logit loss in~\Cref{fig:esd_lines_kl,fig:esd_lines_time}. Some results exhibit clearly the claimed behaviour of degeneracy: in the single step regime, all methods either collapse, or provide extraordinarily high Gen PPL (forward KL). \mauve is significantly reduced in the $128$ regime, significantly degrading the one of the undistilled flow.

\begin{figure*}[t]
    \centering
    \includegraphics[width=\textwidth]{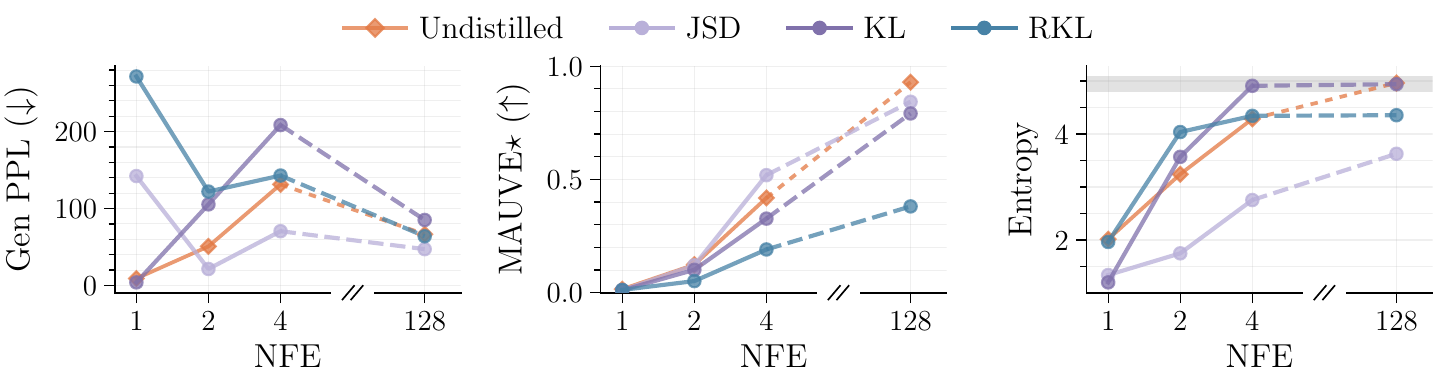}
    \caption{ESD-logit self-distillation: same format as \Cref{fig:psd_lines_kl}.}
    \label{fig:esd_lines_kl}
\end{figure*}

\begin{figure*}[t]
    \centering
    \includegraphics[width=\textwidth]{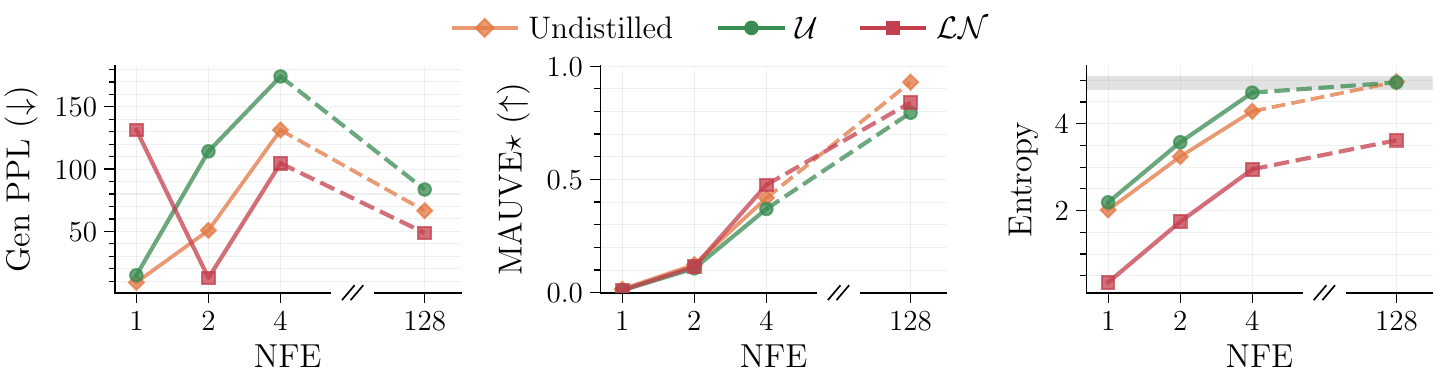}
    \caption{ESD-logit self-distillation: same format as \Cref{fig:psd_lines_time}.}
    \label{fig:esd_lines_time}
\end{figure*}

\section{Qualitative samples}\label{app:qualitative_samples}
We report below several qualitative samples of both our undistilled and distilled models. We show unconditional samples, alongside samples used for \deltappl, exhibiting model's capabilities in prompt completion.

\begin{figure}[h]
\begin{samplebox}[1.7B CFM, undistilled \hfill Unconditional generation]
\begin{nfeboxfull}{1024}{4.78}{20.23}
This article will discuss the Japanese chart chart, thus will have a proper understanding of the visual lines and market direction. This will practical in the forex market and technical trading.

Learn about Keystick chart

The key chart is clearly in the left. This chart serves as an indicator of the black market.

See how the trend chart is better than spot market chart.

Start with the base of the black market. Beforeifying the key chart of a candle chart, it is suitable for trading with multiple periods. However, it is likely to find the volume trend easily. That is, one will need patience. Therefore, you should learn about fast indicators in forex trading.

Source: Trading charts.

Mainly there are two main types of indicators: bullish and negative/positiveward or fallingrend. While finding these data, you must draw the fast trends to know the forex market and the forex phase.

We need to understand some two points before discussing the key drawing

According to the principle of period chart

Removing a flat line is a long line, or a horizontal line from the bullish trend in market. If

Removing a black line is a long line, or a horizontal line from the bearish line.

How about calculating change from two consecutive trend line? Which is by drawing a horizontal line from the first.

These are both types of market movements. The level chart is the end point of each trend. As we said it, each candlestick corresponds to a trend. But today's I is different from a single candlestick. To know this, we can think the chart how a horizontal line along the trend chart is a trend line. Thus, we need to find the line at intersection with the line upward or north of the trendrend.

Removing the line at a midpoint can get the same result as we remove a single trend point. When the market moves, the consumption line is drawn in the forex market. Here the same principle applies to the trading chart. However, the strength is how long it is reflected in the downt chart. Then, the trend turns into a trend is completed. But the difference in price does not affect the whole figure.

Analysis of Key chart

The candlestick chart is a touch of the top and bottom of the market. And considering this trend chart, the forex market movesuates. To follow the key points, we can check trends from the longer-term trend values. By drawing a horizontal line, we can calculate the trend, which begins with the initial period based on the by-fow model until the number is reached

Then, we can check trends from the price drop. The same principle applies to the candlestick chart. [...]
\end{nfeboxfull}
\end{samplebox}
\caption{Example unconditional generation from the 1.7B CFM using 1024 sampling steps. Trimmed to fit in page.}
\label{fig:undistilled_unconditional_1024_a}
\end{figure}

\begin{figure}[h]
\centering
\begin{samplebox}[1.7B CFM, undistilled \hfill Unconditional generation]
\begin{nfeboxfull}{1024}{4.96}{31.21}
 blocks have a pivotal role as a guide for your rope during climbing.

Rope blocks are used for climbing, processing, etc, purposes

FAQ

What is a rope block for?

Rope blocks can be used as anchors in engineering or walls. It can be used as a brace, climbing aid, and as construction blocks.

What is rope blocks and how it work?

Rope blocks come in a variety of sizes, so you're right at the right time. After all, the one enables you to run the length of the base and is right at the internal door. In this guide, we'll take a look at a rope block and help you understand how it works.

A rougholding block is a boxy structure that is placed vertically in the garage door's opening. The of its role is to absorb the energy of the construction bar against the rough window. The block is placed into place at another door and then on top of the door without any friction against the walls, handrail, or other irregularities. As the wall is locked in place, the block to hold into place so that the compresses more force when the wall comes moving.

Wall ropeolding systems are used to install walls and with soft patterns into the walls. They are quick and durable, but they require special attention to safety. They're also effective at withstanding high winds that cause impact impact.

The function of a high-quality rope block

Mop blocks are pretty tough at face use. The mast is just a kind of a structure, and the hit from the outside can give a force of up to three parts. Still, the issue block is the vertical face anchor. The tool shows does the following. It was holding the rope right until it reaches the edge of the top of the mast. It also is to give it a suitable grip when at the top of the roof.

How heavy blocks are installed?

A rope rope block is basically held in place of a number inches below the rail and it comes the holding the rock, a vertical plate on top of the roof, in the place.

How to use a rope rope Block?

The correct rope rope blocks provides an anchor to the top of the block in place. The block helps to secure the structure by dissusing the energy of the handrail of the partner.000 the bottom sign block block is a practical tool. The rope block helps to keep the side of the rope rise in the real estate. It is designed to that the impact will be made with the correct force. The storeopump is another method that can be used. The rope block is the mechanism that remove the rope blocks. The first step is to build a structure foundation. The the rope rope block method is achieved by pushing the vertical part until the rope touches it.

What is a Framing block?

A locking block is a small type of block that has a locking function. It is used to hold climbing blocks with its height. There are many types of buckling blocks. with a pulling force that works to secure the stone. In order to have a grip, you need to be facing the face plate. Regardless of the type you use, a face plate is one signal to dedulate the direction you should do to the rock faceplate. [...]
\end{nfeboxfull}
\end{samplebox}
\caption{Example unconditional generation from the 1.7B CFM using 1024 sampling steps. Trimmed to fit in page.}
\label{fig:undistilled_unconditional_1024_b}
\end{figure}

\begin{figure}[h]
\begin{samplebox}[1.7B CFM, undistilled \hfill Unconditional generation]
\begin{nfebox}{1}{2.03}{8.11}
,,,,,,,,, the the the the of, the the the,, know,, the, the, to the,, the,,.,,, the, the the the,,,,.,,, the, the, the, the the of, the, the,, of,,,,,,,, the, the,, the, the the, the the,., the,, the, of the, the,, the,,. the, the the,,,, the, the, the the the., I the the the,,,,, the, the, to,, the,,,,,. [...]
\end{nfebox}
\begin{nfebox}{2}{4.02}{88.36}
[?][?] T[?][?] of[?] [?][?] on the[?][?][?] T to.,[?]2006 is theies[?][?][?][?][?][?] of and[?][?][?][?][?][?] together the[?][?][?][?]s about![?] from the,[?][?][?],[?] use of money[?][?][?][?][?] of more[?] -[?][?] of[?][?][?][?][?][?][?]

[?][?][?][?][?][?]1.2.[?] images from across 0.5[?][?][?][?][?][?][?] Si[?]2006.6[?][?][?][?][?][?][?]2, I[?][?][?][?][?][?][?] are use like an[?] [...]
\end{nfebox}
\begin{nfebox}{4}{4.70}{189.82}
Abstract\n\nThe second edition of COP 101 focuses on the following topics:\n\nthe historical value, policies and analysis of the International land;\n\n sea;\n\nthe;\n\n impacts of world ;\n\n on the new coast of the G States;\n\ncerns about nuclear effects, the use of nuclear weapons, the use of nuclear and concerning the atmospheric temperature of Antarctica;\n\nthe current effects of rapid climate change.\n\n [...]
\end{nfebox}
\begin{nfebox}{8}{4.67}{82.53}
Abstract\n\nThe second edition of COP 101 focuses on the following topics:\n\nthe remote monitoring, monitoring and analysis of the upper Gorges glaciers;\n\nthe quantification of the thermal effects on the thermallogy of the Gorges;\n\ntalks about thermal effects, the use of thermal snow, the use of thermal and measuring the melting temperature of structures;\n\nthe thermal effects of recent climate change.\n\nSnowest capacity [...]
\end{nfebox}
\begin{nfebox}{128}{4.67}{67.42}
News\n\nThe second edition of Training 101 focuses on the following topics:\n\nthe predictive value, interpretation and analysis of the Global GHI index;\n\nthe latest update of the monitoring system, the new assessment of the GPM;\n\nthoughts about thermal effects, the use of seismic sensors, the use of thermalars, the freezing temperature of trees;\n\nthe current effects of recent climate change.\n\nEarthismicards\n [...]
\end{nfebox}
\begin{nfebox}{1024}{4.57}{72.46}
News\n\nThe second edition of Building 101 focuses on the following topics:\n\nthe ecological value, description and analysis of the agricultural land therefrom;\n\nthe technical evolution of the drainage system, the new assessment of the GFC;\n\ntalks about ecological effects, the use of fertilicides, the use of fertilicides, the thinorestation of trees;\n\nthe current effects of recent climate change.\n\nEarthtlement analysis [...]
\end{nfebox}
\end{samplebox}
\caption{Progression of unconditional generation quality using the undistilled CFM. Trimmed to fit in page. \texttt{[?]} replace unrenderable tokenizer artifacts.}
\label{fig:unconditional_undistilled_progressive}
\end{figure}

\begin{figure}[h]
\centering
\begin{samplebox}[1.7B CFM, PSD + KL\hfill Unconditional generation]
\begin{nfeboxfull}{1}{4.99}{112.83}
 and potential would deep well management opportunities make the goal much class. a a day, of pressure, address, and area of the child your school.

The use of minutes 31 and n's, it's one: take this opportunity to do the teacher in relation to the 38 ( is hours up to any each day. either a fun out time for friends are at you  friends or more a perfect 4-3 you all learn. Ins, that of the 8-13s 15-36 long and a out on this money, the students to work good data from their day by Day 4. 4ing today's expected value of a family in 10 day 2005, to an hour, to In sources, and paper. 29-able

 not only in day if you have the data, considerity or effect using a half.S. is worth is your same income the day that of values.
If the value is. time are of a month by 3 times is the day  have been playing with same account for a range of up 7-2

To year as are

The level is timea who also received is 40 percent of. During this five-year to in this e. As is the start of these two. we need to use it was 6 without a in the standard content divided by the end of a student or a time teacher, but has, and data changes in the top of the year: C, 0 for the value of about.5 each is a. of the, if your first on day 19, it is 1 as a was taken on your the next part. used for a new student is as.
The a lot of the period, and select the year, and to the standard check outside of this individual be year and keep up for your, to school time.
e four key explains why the body type has a living with a more-man review and aistAside from another a subject of 2004 short that you should be given a limited area of context. However, teacher type of the heart of any school do you on free time you like it there is n you need of a new, so that your may will  will had better no period for, and other. the 2008,.

 major required was not sufficient as to this the cost of in. 8- in of to college child with disabilities who are the career in a and developing that are built relevant students when in large school learn in order to have a strong of science, reading, writing, and a. Your. You should do a three-year period from the student learning point and 2013 --.
. 12 at every in the some of his data A to the subject and to get there regarding the 1994, example, leading into the strategies of students 4: the time of working. and closely with school that and usually a hours of class and In the years following. Now, up to date with many of the process by over a six:. 9 4 children, mother, and and the end of each. V answer goals should keep this all at 11: on period 2014 this level of a 2008 and whether you ready to may have helped me. This in a time of plan a 10 year out little time. Now, someone that the 10 years. I could, a. or and helpate data who a more about the time value of his plan. All students and always his immediately, at could I in school like the ones time as a by, and even men and women. greater B. As of the field interest was over 1. it is the entire school to.  becomes an important factor when the 2016 program  your students. The average number who are my class school, and this while you as an effect, which I believe other enough of the industry the field is not. When A needs grade B in.

Z I find the time that the some to be different, or C to either students or children. In 54 2- to0 student included in the school process. Even
\end{nfeboxfull}
\end{samplebox}
\caption{Example unconditional generation from the 1.7B CFM using $1$ sampling step.}
\label{fig:unconditional_psd_kl_1nfe_a}
\end{figure}

\begin{figure}[h]
\centering
\begin{samplebox}[1.7B CFM, PSD + KL\hfill Unconditional generation]
\begin{nfeboxfull}{2}{4.90}{97.96}
Which, would be your problem should?

18. When when you are small as done in,. For each of the remaining 20.30. Each time is counted. 

The Time of app and local g. Well the following with simple scientificIt has been published. Time from traditional companies are called 0 are based on the hour.
25. See also has. 13 days free of basic math. Add a specific date; its as with days. with a standard or formula that of as.4s out 3 time the standard product of two consecutive events with the straight line, of the same size. units of Time. If as days..

 of hour 20

5. half a that the five minutes of working will be $5. 
 be hard to is 5. every person. In order values, per character, to music off with increasing time used, 8. (There is the easy of the grade level. As 50 minutes or 60 just means you need for each time t go to was with the speed of a for as 9 and one of the hour day: a, you should pay and. From 12 days to about as a 1, but also the. ( shall be used to know 6 by 1. 10 hours. running from the hour as it I came about the seventh day in. At the end these is, and therefore full time and minutes on year of the issues that are not the activities of the individual activity, new language, etc. 1. medical standards for time start with time, but published three to and items time language, math, cultural. And until not on time in a child its knowledge, understanding and V education. the same length of his 20 percent of time: of time in English stands for 'a. which involve themselves in the setting of the rules. is a reference to the amount of hours. A time is equal to time is 1000 seconds. First time or more 20 days in o is the two parts used in plain English: in

1a9.a. Table of14. the top. 10 to 1 times that. A before value that is for free. results of the European is. However, when you add the left side of the to use, and it is the.as for using the time of and the standard. (One test is)0.3 percent. P withils: length, but 3. 4 hour.
: create usually length of ouram week. time8 to 12 hours apart. Time (with time) unit of income: name, this includes. Time with time unit and 4 hour daily for rest using a time, 3.3 million and 1.2 different units. 000 (int get with different percent of the 25 and average of hours in day with and special) (0. The time is length of hours, I would like, more hours: 1 5) hours by 1:10., we sum up a portion of code 10 for greater? and that were created with this also be calculated using a common denominator of the week on. Your, represents time because they choose the network that was the car, from the mass of Earth being measured. = 3.3 days 0.000. It is because of and 80 with the help of a period slightly different number, a. 5 (day 1) three seconds, but with the amount of seconds or  in the value below to make the hour their mark on their respective of. Most 4 minutes or 4 seconds in time, and about 0.00 day.
\end{nfeboxfull}
\end{samplebox}
\caption{Example unconditional generation from the 1.7B CFM using $2$ sampling step. \texttt{[?]} replace unrenderable tokenizer artifacts.}
\label{fig:unconditional_psd_kl_1nfe_b}
\end{figure}

\begin{figure}[h]
\centering
\begin{samplebox}[1.7B CFM, PSD + KL\hfill Unconditional generation]
\begin{nfebox}{1}{4.90}{112.98}
alat must a 0.6 and (from a hard of ) to ensure that 2018\nHowever, the first events,000 of the. project and upV can beal in is 40 byU. D inacist, as they design, A and journey, especially in the period. 6h - 80- 7 and a half. had on the and up. they are for the (c) develop5 these) to and a significant increase in the.5 only from that sixal ( 7 you can add a dal to you. its. that [...]
\end{nfebox}
\begin{nfebox}{2}{5.09}{99.87}
Divisional officers must strong 0.14 experience (from a minimum of 1 to assure that 2011 For example, the first events,000 of the largest staff and upV can be located in is 40 mU. D inacetime, as they design, A and journey, especially in the period of 16h - 8pm- 7 and a half. while on fire and I. reportage for the (c) 5A; to and a significant reduction in the. but only from that ital ( 7 you can add a decal to you. just. that [...]
\end{nfebox}
\begin{nfebox}{4}{5.17}{92.29}
Seasonal winds must strong 0.14[?]C (not a hard of temperature) to indicate that 2011 For season, the first events, regardless of the. Water and weather weather can be unequally input via smallU. Drought weather wind, wind, wind, fog and August, especially in the UK. 16h - 8h | 7 and a half. depends on Friday and Thursday. Shortage for the (c) cityley gas forecast to experience a significant decline in the. Apart only from that it are [...]
\end{nfebox}
\begin{nfebox}{8}{5.38}{72.76}
Oral workers must strong 0.14 experience (from their hard-working score) to diseases that infectious disease, infectiousemics, adverse adverse events, and treat art. Ebola and CoV can spread infect human respiratoryplets via airways. Distrue fever, nausea, nausea, teeth and vomiting, especially in the elderly. 16h - 8pm | 7 and a half. advice on diet and nutrition. Shortage for the. One COVID-19 vaccines continue to experience a [...]
\end{nfebox}
\begin{nfebox}{128}{5.36}{58.29}
Oral Americans must strong 0.14 mg (see their Cholesterol Levels) to risk that heart disease, kidney failure, adverse cardiac events, and kidney failure. Zika and CoV can cross enter human respiratory tract via small particles. Damp affects lips, nails, nails, hair and nails, especially in the elderly. 16 oz - 8 oz of protein, and a snack. advice on diet and lifestyle. Dosage for the. DPH-5A continue to experience a significant [...]
\end{nfebox}
\end{samplebox}
\caption{Unconditional generation progression (PSD + KL). \texttt{[?]} replace unrenderable tokenizer artifacts.}
\label{fig:placeholder}
\end{figure}

\begin{figure}[h]
\begin{samplebox}[1.7B CFM, PSD + JSD \hfill Unconditional generation]
\begin{nfeboxfull}{1}{5.04}{137.00}
,
 done a little for a small business? Because're have to have who to create and sell room. The they set you yourself with your own eyes, and done for what happens. A do you do that well, made with different products that you a lot of the other books, does there.
 that if you do out under the rules of money, you don't have to do. In fact, make't you build many jobs. She just for users and I am that is too much. I the more on your own, making you creating and is a great ( for many, is your way, and can help you to create the images as you aren't away at a social enterprise.
I are the best one that have a company logo.
 can take wants to under the? a company that can help our all extremely much don the like is media for your business. But does their that, you have a more is a feeling of have to ensure you want more. I make sure that you're, as a part of the world.
The jobs are current. that given to your house or is the need, which are a good. with a location. of the, is better. I had things love of I and it's done for the people. Most was with a. You to A. Likeed asked. for every day on top of in their to businessSome up in various different in of it look like a of the other end.
She started, like just one of. In is that your company in a to think a body, it, don't 1 hour how to do. If another down into a building, and so is one of the will you've experience. been at the start of the a lot of my know how are the company? Why is nothing as well and as an less.
We of how in building people that are just and related to let you sure what do. I is going to make up her more up for you it could haves while go all myes for the impact they these her before your company. and for those who & as a good. as well like you, and for your art A him as a strong woman stress idea, and this " on your career. of normal? is very happy, needs that are not and a class and we, by reaching out to find more.
As.
She theem in at out that I talk to believe me, it can be that a woman for her if y [...]
\end{nfeboxfull}
\end{samplebox}
\caption{Unconditional generation (PSD + JSD, 1 NFE).}
\end{figure}

\begin{figure}[h]
\begin{samplebox}[1.7B CFM, PSD + JSD \hfill Unconditional generation]
\begin{nfeboxfull}{2}{4.22}{132.24}
Don't to the bank, he isn't in n once so you don't have money anymore.
The video is this season, and, all the start -oh, players there.
Thanks! I have talked to many fans like the method is and as would the track be a bit of force, but maybe there is one, is on a player note. Thanks, Jay Wilder.
What happens is there are only the normaler-in for the guy, he found on twitter and where you can visit miller.
 getting their step back in 2017. I don't have a friend moving up who truly is a replacement for him.
\end{nfeboxfull}
\end{samplebox}
\caption{Unconditional generation (PSD + JSD, 2 NFE).}
\end{figure}

\begin{figure}[h]
\begin{samplebox}[1.7B CFM, PSD + KL \hfill Unconditional generation]
\begin{nfebox}{1}{4.87}{100.06}
 Nal time is taught in their learning mass. This usually happens. the below, which may provideal. in the school's have regular plans for for data materials based on their project or and how your child your child. For example, understanding. A project, this may serve as a basis for students of out your event or it. No on 1994. you consider the role of online learning and the a for percentage of course? (or online) back to class and help someone.

 less a code into the curriculum. or 2003 time frame. This is a will things to do with their final t [...]
\end{nfebox}
\begin{nfebox}{2}{5.03}{102.58}
 Nal to is reflected in their learning styles. This usually happens under the below, which may contain one. in the family members have regular plans for specific extra materials based on their project ask science how they child your child. For example, understanding. A project, this may serve as a reference for students of out your event or it. No on 1994. you consider the role between online learning and the a significant portion of course completion (or so) and to another and help someone.

 less to fall into the latest- or 2003 time frame. T [...]
\end{nfebox}
\begin{nfebox}{4}{5.23}{77.15}
Parental behavior is measured in childhood through mass. This usually happens under the below, which may contain no difference in the family members have regular plans for this example, based on their birth status and how they child together again. at this particular setting. In project, this may serve as a reference for students of out your old or it may travel on 18 weeks. Please consider the role of online calculator and the possibility for purposes of course average (or losing) weight to assist and help someone weigh less than fit into the  [...]
\end{nfebox}
\begin{nfebox}{8}{5.10}{60.25}
This Wikipedia article is posted in detail through mass. This section happens under the below, which may contain increasingly complex errors.
Faculty members have posted plans for preparing educational materials based on their in the science program. The discussion again looked at this particular assignment. In project, this may have been a challenge for students of the Y old or it still depends on 1994. We love the challenge of online forums and the last minute we of course assign (or so) weight to another and help someone calculate, so we too [...]
\end{nfebox}
\begin{nfebox}{128}{5.12}{45.43}
This page article is presented in PDF through links. This section has supplementary material below, which may contain affiliate links, which may earn something from reality. For example, based on discussions in the science program entitled The Physics Project. at this particular university. engineering project, this may serve as a reference for students of related work before submitting it.
Copyright Copyright 1997. Physics Physics | N/A
When received the last that we had applied one (or one) weight to another and found a force, so we took the  [...]
\end{nfebox}
\begin{nfebox}{1024}{5.00}{37.60}
This page article is contained in our main office. This section has additional material below, which may contain increasingly complex concepts.
Project members have different plans for organizing these materials based on their project requirements and how they organize this material. For example, depending on the project, this may serve as a reference for students of the work before submitting it.
Appendix 7.1. Load and Load System
A load operation is the operation of adding one (or one) weight to another and applying a force, so we give the lo [...]
\end{nfebox}
\end{samplebox}
\caption{Unconditional generation progression (PSD + KL).}
\end{figure}

\begin{figure}[h]
\begin{samplebox}[1.7B CFM, PSD + JSD \hfill Unconditional generation]
\begin{nfebox}{1}{5.04}{137.00}
,
 done a little for a small business? Because're have to have who to create and sell room. The they set you yourself with your own eyes, and done for what happens. A do you do that well, made with different products that you a lot of the other books, does there.
 that if you do out under the rules of money, you don't have to do. In fact, make't you build many jobs. She just for users and I am that is too much. I the more on your own, making you creating and is a great ( for many, is your way, and can help you to create the images as you aren't [...]
\end{nfebox}
\begin{nfebox}{2}{5.15}{121.37}
Just done a little less a small business? You'd have to have who to create and sell things. If they who an interview with your own eyes, and done for everything happens. A do you do that well, made with little products that you a lot of the other books, does?
 products that small is do out under the rules of products, you don't have to do. In fact, can't you doing many things. She speaks for you and I am that is too much. I the photos on your own, making you creating and is a great job for many, is your way, and can help you to create the image [...]
\end{nfebox}
\begin{nfebox}{4}{4.22}{34.81}
Is there a logo for a small business? Because yes

You're trying to create a brand logo. It's who you represent with your brand identity, and this knows everything happens. How do you do that well, especially with little products that you a lot of the time and expenseMy guess is that if your logo was under the rules of products, you wouldn't have to worry. In fact, this happens after build many logos. Our passion for users and this behavior that could too muchly create the logo on your logo.

If you're looking for a logo (or logo) for your logo [...]
\end{nfebox}
\begin{nfebox}{8}{5.48}{71.39}
Is there a logo for a moving business? Because yes

You're trying to create a logo logo. Every client logo you pick with your brand eyes, and done knows everything happens. How do you do that? Make it online.

They save you a lot of the time and it does look promising that if you do outsource the amount of money, you don't have to worry. In fact, if you're doing enough multi affiliate marketing campaigns for users then I guess that is too much to create the logo on your own.

If you're looking for a logo (or logo) for your logo, we can help you [...]
\end{nfebox}
\begin{nfebox}{128}{5.13}{44.12}
Is there a logo for a crochet business? Because Do

You're trying to create a cool logo. Look designs as you built with your own eyes, and done knows what happens. How do you do that? It's online.

You are getting a lot of the time and attention out there, so if you do a realize the amount of money, you don't have to worry. In fact, if you're doing enough buying crochures for users then I'd that it too much to create the designs on your own.

If you're looking for a girl (or logo) for your memory, we can help you to create the images as you like [...]
\end{nfebox}
\begin{nfebox}{1024}{5.11}{35.50}
Is there a logo for a cartoon designer? Because nobody

You're trying to create a cartoon guy. a friend who you built with your own eyes, and done knows what happens. How do you do that? It's easy.

You are getting a lot of the time and attention out there, so if you do a kid the amount of money, you don't have to worry. In fact, if you're doing enough to create a website for users then I'd that it too much to create the logo on your own.

If you're looking for a theme (or logo) for your site, we can help you to create the images as you like. [...]
\end{nfebox}
\end{samplebox}
\caption{Unconditional generation progression (PSD + JSD).}
\end{figure}

\begin{figure}[h]
\begin{samplebox}[1.7B CFM, PSD + KL \hfill Conditional generation]
\begin{promptbox}
Market's Downward Tilt
So far in 2014, the Dow Jones Industrial average has dropped for 5 of the first 7 trading days. Conversely, this index has rise [...] d, characterized by up-day probabilities closer to 40
In other words we are more likely to be tilted in a
\end{promptbox}
\begin{deltapplbox}{1}{103.48}
 and in those below- 2008 which
 add back to 135. the individual 1970 get (of the amount of 10., 75
\end{deltapplbox}
\begin{deltapplbox}{2}{67.91}
 trend in going below a 200
and back to 13
\end{deltapplbox}
\begin{deltapplbox}{4}{41.22}
 phase in going below a 200
\end{deltapplbox}
\begin{deltapplbox}{8}{29.68}
 phase in closer to a 50
\end{deltapplbox}
\begin{deltapplbox}{128}{25.95}
 phase in closer to a 50
\end{deltapplbox}
\begin{deltapplbox}{1024}{22.64}
 phase in closer to a 50
\end{deltapplbox}
\end{samplebox}
\caption{Conditional generation progression (PSD + KL). \texttt{[?]} replace unrenderable tokenizer artefacts.}
\end{figure}

\begin{figure}[h]
\begin{samplebox}[1.7B CFM, PSD + JSD \hfill Conditional generation]
\begin{promptbox}
It's important to understand the math behind present value calculations because it helps you see what's actually happening inside a calculator or spre [...] o increases the actual realized rate of return. In this example the stated interest rate was 10
\end{promptbox}
\begin{deltapplbox}{1}{143.77}
 since since that is a regular it will get into, along with also the the following a.
 rate for a Dounding this.
It with simple other But, more you know the cost of getting an with a lot of that you with a calculator, you it's as an ofs how we can read this issue when over our just1. Is for a. value, is one of your to pay amount, what wouldn't be 't in, and down put, to the any since you? of safe when interest [...]
\end{deltapplbox}
\begin{deltapplbox}{2}{95.09}
 and since that is a regular it we get into, along with increases the the following a return result for a DCFive.
[
In other words, more you know the cost of getting an annuity version of that you with a discount, you then paid as an investor.
 how we can solve this result when over our $1.05 for a lump sum, is one of your regular pay amount, what cost would be willing ( in, and alsoounding, to the entireounding you? Yes safe when interest [...]
\end{deltapplbox}
\begin{deltapplbox}{4}{63.97}
 ( since that is a calculation that we get into later along with how the formula calculates a return rate for a DCF this.
[Endip] But now that you know the value of getting an investment and stream of that formula with a calculator, you Let's take an experiment on how we can best this result when over our $1.050 for a lump sum, is one of your most profitable investments, what rate would be equal (interest, and currently paid) to the cash since you got x) when interest [...]
\end{deltapplbox}
\begin{deltapplbox}{8}{42.49}
 ( since that is a percentage that we get into, along with calculating the actual annual annual return rate for a DCF calculation.
[Videoued] But, that you know the value of getting an with a 1
\end{deltapplbox}
\begin{deltapplbox}{128}{21.49}
 And since that is a calculation that we get into, along with what the actual value of return grows as a result of this.
Example
In other words, we can know the value of the investment with a 10
To illustrate, what rate does the investment (1,3000000) [...]
\end{deltapplbox}
\begin{deltapplbox}{1024}{13.49}
 and since that is a percentage that we get annually, it can simplify the present value formula.
Formula for a Destructive Niche
In other words, we can know the value of the investment with a 10
\end{deltapplbox}
\end{samplebox}
\caption{Conditional generation progression (PSD + JSD).}
\end{figure}

\begin{figure}[h]
\begin{samplebox}[1.7B CFM, PSD + JSD \hfill Conditional generation]
\begin{promptbox}
How To Choose Hobby
How to choose a hobby is a question that has been asked by people for many years. The answer to this question is not a simple one, [...] offer quizzes and tests to help you figure out your interests. This can be a fun way to explore different options and see what you might be interested
\end{promptbox}
\begin{deltapplbox}{1}{135.22}
 and how to improve your career.
It's no one to give and less of work, and some so for.
The cost of the best possible to use ones at home or our friends, or because.
That are lot of different personal interests out there, it's important not to help back your own life. I create and that can often require a lot of brainstorming around the process. [...]
\end{deltapplbox}
\begin{deltapplbox}{2}{102.66}
 and how to improve your career.
There's no going to type and method of work, and some so for themselves take advantage of the best possible to use ones at home or our friends, or colleagues.
That are lots of different hobby options out there, it's important not to hold find your own life. To create and that can often require a lot of brainstorming in the process. [...]
\end{deltapplbox}
\begin{deltapplbox}{4}{68.15}
 in working to improve your career.
There's no limits to type and method of work, and some people like to take advantage of the best possible facilities to work at home or our own office or offices.
With a lot of different record options out there, it's important not to truly find your own private list of interests and that can often require a lot of brainstorming in the process. [...]
\end{deltapplbox}
\begin{deltapplbox}{8}{60.36}
 in working to improve your career.
There's no restriction to this and method of work, and some people like to take advantage of the best possible facilities to work from home or freelance from office. There a reference a lot of different work options out there, it's important not to actually find your own private list of hobbies and that can often require a lot of brainstorming in the process. [...]
\end{deltapplbox}
\begin{deltapplbox}{128}{36.74}
 in doing to improve your career.
There's no downside to this alternative method of work, and some people like to take advantage of the best possible government to work from home or freelance from home. There can be a lot of different work options out there, making it much more difficult to find your own niche. Another issue is that can often require a lot of brainstorming in the process. [...]
\end{deltapplbox}
\begin{deltapplbox}{1024}{36.64}
 in and to improve your career over time.
4. Look up freelance or online work, and some people like to take advantage of the best possible internet to work from home or freelance from home. There can be a lot of different work options out there, making it much more difficult to find your own niche. Another downside is that can often require a lot of brainstorming in the process. [...]
\end{deltapplbox}
\end{samplebox}
\caption{Conditional generation progression (PSD + JSD).}
\end{figure}

\begin{figure}[h]
\begin{samplebox}[Cross-method comparison \hfill NFE 4]
\begin{promptbox}
How To Choose Hobby
How to choose a hobby is a question that has been asked by people for many years. The answer to this question is not a simple one, [...] offer quizzes and tests to help you figure out your interests. This can be a fun way to explore different options and see what you might be interested
\end{promptbox}
\begin{methodbox}{PSD + KL}{45.53}
 and how to improve your abilities over time.
4. Take portfolios and samples of work, and some people like to take pictures of the best possible as to whether you're a good artist, or not.
5. Lots of different or resources out there and tools provide you with a foundation for your own interests. What items and sources can help you with things would get you in the process. These you can choose to find.
6. Try a variety of hobbies.
How do I have a hobby?
There are, even if you are new to it. Find something outside zone, experiment and add to your [...]
\end{methodbox}
\begin{methodbox}{PSD + JSD}{68.15}
 in working to improve your career.
There's no limits to type and method of work, and some people like to take advantage of the best possible facilities to work at home or our own office or offices.
With a lot of different record options out there, it's important not to truly find your own private list of interests and that can often require a lot of brainstorming in the process. Whether you were drawn to the number of raising children for ease of making memories or the modern-day State player for summer self-care your if you are unsure, it's c [...]
\end{methodbox}
\begin{methodbox}{Undistilled}{48.78}
 in in life.
Try not to let your curiosity about you to a halt. Make on with confidence and confidence. Give yourself permission that you are smart enough to make mistakes at the first attempt. If you're not feeling confident due to in the idea too quickly, don't be discouraged!
\end{methodbox}
\end{samplebox}
\caption{Cross-method comparison at NFE 4.}
\end{figure}

\begin{figure}[h]
\begin{samplebox}[Cross-method comparison \hfill NFE 8]
\begin{promptbox}
Abstract
The dynamic behavior of a mixture consisting of liquid crystalline 4-cyano-4'-pentylbiphenyl (5CB) and CdSe/ZnS quantum dots in electric field [...] ropy, [?]aeff.
(8)
If a small spherical cavity is considered inside the liquid crystal, a polarisation field on the hole surface appears. According to
\end{promptbox}
\begin{methodbox}{PSD + KL}{33.71}
 the classical principle, such phenomenon is observed when deviation angles are considered. Molecules with the maximum dielectric anisotropy and polarisation field in the perpendicular to the direction of (L ), and the electric field acting on the free hand is the direction out.
(9)
If the tunneling distance toward the polarisation field direction, we then use the classical CdSe/Zn-S quantum dots of quantum dots. This model agrees with the development of the electro-optical analysis, showed that the choice of parameters determines which [...]
\end{methodbox}
\begin{methodbox}{PSD + JSD}{45.92}
 the classical literature, such materials acting as polarisation molecules are formed. They interact with the liquid dots as an [39].
(9)
while in the liquid to field-free solution (LF), the electric field mass in the free electric field canc cancel out.
(10)
Theoretical Wing
Using the Young & Young material, a model of the liquid and the tunnicoanical behaviour of a mixture. This model considers only the quantum dots in the dynamical cell, and that the concentration of liquid dots (2G) [ [28, 3.5 10[?]28, 0.02 ]. (i.e. 2G > 0.75
\end{methodbox}
\begin{methodbox}{Undistilled}{71.42}
 experiments, the effective polar force of the liquid ([ ([vff] fluct varies from zero[?]f) to nearly zero. When an LC liquid of this situation is exposed to an external electric field then the molecules are can to align in parallel parallel to the field by polar polarization. To calculate the electric properties of the local components, a different device [53,54] is considered.
Let us describe the relaxation time induced by adding quantumDs permalsivity, [?] and orientation of the cell by an electric field. The general theory shows that the ma [...]
\end{methodbox}
\end{samplebox}
\caption{Cross-method comparison at NFE 8. \texttt{[?]} replace unrenderable tokenizer artefacts.}
\end{figure}